\documentclass{article}

\usepackage[final]{neurips_2021}

\usepackage[utf8]{inputenc} 
\usepackage[T1]{fontenc}    
\usepackage{hyperref}       
\usepackage{url}            
\usepackage{booktabs}       
\usepackage{amsfonts}       
\usepackage{nicefrac}       
\usepackage{microtype}      
\usepackage{xcolor}         
\usepackage{xspace}
\usepackage{bm}
\usepackage{subfigure}
\usepackage{graphicx}
\usepackage{amsmath}
\usepackage{wrapfig}
\usepackage{amsthm}
\usepackage[ruled, vlined]{algorithm2e}
\usepackage{natbib}
\setcitestyle{numbers,comma,open={[},close={]}}

\newcommand{\REM}[1]{}

\DeclareMathOperator*{\argmin}{argmin} 
\DeclareMathOperator*{\argmax}{argmax} 
\newtheorem{theorem}{Theorem}
\newtheorem{lemma}[theorem]{Lemma}
\newtheorem{corollary}[theorem]{Corollary}

\newtheorem{innercustomthm}{Theorem}
\newenvironment{customthm}[1]
  {\renewcommand\theinnercustomthm{#1}\innercustomthm}
  {\endinnercustomthm}
\newtheorem{innercustomthml}{Lemma}
\newenvironment{customthml}[1]
  {\renewcommand\theinnercustomthml{#1}\innercustomthml}
  {\endinnercustomthml}
\newtheorem{innercustomthmb}{Corollary}
\newenvironment{customthmb}[1]
{\renewcommand\theinnercustomthmb{#1}\innercustomthmb}
 {\endinnercustomthmb}
\title{A Surrogate Objective Framework for Prediction+Optimization with Soft Constraints}

\newcommand{\eg}{\emph{e.g.,}\xspace}
\newcommand{\ie}{\emph{i.e.,}\xspace}
\newcommand{\etaldot}{\emph{et al.}\xspace}

\author{
  Kai Yan\thanks{Contributed during internship at Microsoft Research. $\dagger$ Corresponding Author.} \\
  Department of Computer Science\\
  University of Illinois at Urbana-Champaign\\
  Urbana, IL 61801 \\
  \texttt{kaiyan3@illinois.edu} \\
   \And
   Jie Yan \\
   Microsoft Research \\
   Beijing, China \\
   \texttt{jiey@microsoft.com} \\
   \And
   Chuan Luo \\
   Microsoft Research \\
   Beijing, china \\
   \texttt{chuan.luo@microsoft.com} \\
   \And
   Liting Chen$^*$ \\
   Microsoft Research \\
   Beijing, China \\
   \texttt{98chenliting@gmail.com} \\
   \And
   Qingwei Lin$^\dagger$\\
   Microsoft Research \\
   Beijing, China \\
   \texttt{qlin@microsoft.com} \\
   \And
   Dongmei Zhang \\
   Microsoft Research \\
   Beijing, China \\
   \texttt{dongmeiz@microsoft.com}
}

\begin{document}

\maketitle
 
\begin{abstract}
 Prediction+optimization is a common real-world paradigm where we have to predict problem parameters before solving the optimization problem. However, the criteria by which the prediction model is trained are often inconsistent with the goal of the downstream optimization problem. Recently, decision-focused prediction approaches, such as SPO+ and direct optimization, have been proposed to fill this gap. However, they cannot directly handle the soft constraints with the $max$ operator required in many real-world objectives. This paper proposes a novel analytically differentiable surrogate objective framework for real-world linear and semi-definite negative quadratic programming problems with soft linear and non-negative hard constraints. This framework gives the theoretical bounds on constraints' multipliers, and derives the closed-form solution with respect to predictive parameters and thus gradients for any variable in the problem. We evaluate our method in three applications extended with soft constraints: synthetic linear programming, portfolio optimization, and resource provisioning, demonstrating that our method outperforms traditional two-staged methods and other decision-focused approaches. 
\end{abstract}
\section{Introduction} 

Mathematical optimization (a.k.a. mathematical programming), e.g., linear and quadratic programming, has been widely applied in decision-making processes, such as resource scheduling \cite{LuoEtAl20}, goods production planning \cite{goodsprogramming}, portfolio optimization \cite{surrogate-melding}, and power scheduling \cite{DPpred+op}. In practice, problem parameters (e.g., goods demands, and equity returns) are often contextual and predicted by models with observed features (e.g., history time series). With the popularity of machine learning techniques and increasing available data, prediction+optimization has become a normal paradigm \cite{pmlr-v28-ganeshapillai13}. 

Prediction becomes critical to the performance of the full prediction+optimization workflow since modern optimization solvers (e.g., Gurobi \cite{gurobi} and CPLEX \cite{cplex}) can already efficiently find optimal solutions for most large scale optimization problems. Traditionally, prediction is treated separately as a general supervised learning problem and learned through minimizing a generic loss function (e.g., mean squared error for regression). However, studies have shown that minimization of the fitting errors does not necessarily lead to better final decision performance \cite{ijns-Bengio97, aaai-WilderDT19, elmachtoub2020smart, task-opt-nips2017}.

Recently, a lot of efforts on using optimization objective to guide the learning of prediction models have been made, which are \textit{decision-focused} when training prediction models instead of using traditional prediction metrics, e.g. mean squared error losses. For linear objectives, the `Smart Predict then Optimize' (\cite{elmachtoub2020smart}) proposes the SPO+ loss function to measure the prediction errors against optimization objectives, while direct optimization (\cite{pmlr-v48-songb16}) updates the prediction model's parameters by perturbation. For quadratic objectives, OptNet \cite{optnet, amos2019differentiable} implements the optimization as an implicit layer whose gradients can be computed by differentiating the KKT conditions and then back propagates to prediction neural network. CVXPY \cite{cvxpy} uses similar technologies with OptNet but extends to more general cases of convex optimization. However, all the above state-of-the-art approaches do not contain the soft constraints in their objectives. In this paper, We consider linear `soft-constraints', a penalty in the form of $\max(z, 0)$, where $z=Cx-d$ is a projection of decision variable $x \in \mathbb{R}^n$ and context variables $C \in \mathbb{R}^{m \times n}, d \in \mathbb{R}^m$. Such soft constraints are often required in practice: for example, they could be the waste of provisioned resources over demands or the extra tax paid when violating regulatory rules. Unfortunately, the $\max(\cdot, 0)$ operator is not differentiable and thus cannot be directly handled by these existing approaches. To differentiate soft constraints is a primary motivation of this paper.

In this paper, we derive a surrogate objective framework for a broad set of real-world linear and quadratic programming problems with linear soft constraints and implement decision-focused differentiable predictions, with the assumption of non-negativity of hard constraint parameters. The framework consists of three steps: 1) rewriting all hard constraints into piece-wise linear soft constraints with a bounded penalty multiplier; 2) using a differentiable element-wise surrogate to substitute the piece-wise objective, and solving the original function numerically to decide which segment the optimal point is on; 3) analytically solving the local surrogate and obtaining the gradient; the gradient is identical to that of the piecewise surrogate since the surrogate is convex/concave such that the optimal point is unique.

Our main contributions are summarized as follows. First, we propose a differentiable surrogate objective function that incorporates both soft and hard constraints. As the foundation of our methodology, in Section \ref{sec:methodology},  we prove that, with reasonable assumptions generally satisfied in the real world, for linear and semi-definite negative quadratic programming problems, the constraints can be transformed into soft constraints; then we propose an analytically differentiable surrogate function for the soft constraints $\max(\cdot, 0)$. Second, we present the derived analytical and closed-form solutions for three representative optimization problems extended with soft constraints in Section 4 -- linear programming with soft constraints, quadratic programming with soft constraints, and asymmetric soft constraint minimization. Unlike KKT-based differentiation methods, our method makes the calculation of gradients straightforward for predicting context parameters in any part of the problem. Finally, we apply with theoretical derivations and evaluate our approach in three scenarios, including synthetic linear programming, portfolio optimization, and resource provisioning in Section 4, empirically demonstrate that our method outperforms two-stage and other predict+optimization approaches.

\section{Preliminaries}\label{sec:formulation}

%
\subsection{Real-world Optimization Problems with Soft Constraints}
\label{sub:formulation-opt}

Our target is to solve the broad set of real-world mathematical optimization problems extended with a $max(z, 0)$ term in their objectives where $z$ depends on decision variables and predicted context parameters. In practice, $max(z, 0)$ is very common; for example, it may model overhead of under-provisioning, over-provisioning of goods, and penalty of soft regulation violations in investment portfolios. We call the above $max(z, 0)$ term in an objective as \textit{soft constraints}, where $z \leq 0$ is allowed to violate as long as the objective improves.

\par The general problem formulation is 
\begin{equation}\label{eq:prob}
\small
\max_x g(\theta, x)-\alpha^T\max(z, 0),\ 
z = Cx-d,\
\text{s.t. } Ax\leq b, Bx=c,\ x\geq 0
\end{equation}
where $g$ is a utility function, $x\in \mathbb{R}^{n}$ the decision variable, and $\theta\in\mathbb{R}^{n}$ the predicted parameters. 

Based on observations on a broad set of practical problem settings, we impose two assumptions on the formulation, which serves as the basis of following derivations in this paper. First, we assume $A\in\mathbb{R}^{m_1\times n}\geq 0$, $b\in\mathbb{R}^{m_1}\geq 0$, $B\in\mathbb{R}^{m_2 \times n}\geq 0$ ,and $c\in\mathbb{R}^{m_2}\geq 0$ hold. This is because for problems with constraint on weights, quantities or their thresholds, these parameters are naturally non-negative.
\REM{In the objective, $g(\theta, x)$ is a linear or semi-definite negative quadratic function} 
Second, we assume the linearity of soft constraints, that is $z=Cx-d$, where $z\in\mathbb{R}^{m_3}$, $C\in\mathbb{R}^{m_3\times n}$, and $d\in\mathbb{R}^{m_3}$. This form of soft constraints make sense in wide application situations when describing the penalty of goods under-provisioning or over-provisioning, a vanish in marginal profits, or running out of current materials. 

\par Now we look into three representative instances of Eq.\ref{eq:prob}, extracted from real-world applications.

\textbf{Linear programming with soft constraints}, where $g(x,\theta)=\theta^Tx$. The problem formulation is
\begin{equation}\label{eq:prob1}
\small
 \max_x \theta^Tx-\alpha^T\max(Cx-d, 0),
\text{s.t. } Ax\leq b, Bx=c, x\geq 0
\end{equation}
where $\alpha \geq 0$. Consider the application of logistics where $\theta$ represents goods' prices, \{$A$, $B$, $C$\} are the capability of transportation tools, and \{$b$, $c$\} are the thresholds. Obviously, $A, b, B, c \geq 0$ hold.

\textbf{Quadratic programming with soft constraints}, where $g(\theta, x)=\theta^Tx-x^TQx$. One example is the classic minimum variance portfolio problem \cite{portfolio} with semi-definite positive covariance matrix $Q$ and expected return $\theta$ to be predicted. We extend it with soft constraints which, for example, may represent regulations on portions of equities for some fund types. Formally with $\alpha \geq 0$, we have:
\begin{equation}
\small
\begin{aligned}\label{eq:prob2}
& \max_{x}\theta^Tx-x^TQx-\alpha^T\max(Cx-d, 0),
\text{s.t. } Bx=c,~x\geq 0.
\end{aligned} 
\end{equation}

\textbf{Optimization of asymmetric soft constraints.} This set of optimization problems have the objective to match some expected quantities by penalizing excess and deficiency with probably different weights. Such formulation represents widespread resource provisioning problems, e.g., power\cite{DPpred+op} and cloud resources\cite{LuoEtAl20}, where we minimize the cost of under-provisioning and over-provisioning against demands. Formally with $\alpha_1, \alpha_2>0$, we have:
\begin{equation}
\small
\begin{aligned}\label{eq:prob3}
& \max_{x}-(\alpha_1^T\max(Cx-d, 0)+\alpha_2^T\max(d-Cx, 0)),\text{ s.t. }Bx=c,~x\geq 0.
\end{aligned}
\end{equation}
In this paper, we consider a challenging task where $C$ is a matrix to be predicted with known constants $d$. In reality, the $Cx-d$ term may represent the "wasted" part when satisfying the actual need of $d$. \REM{ Usually in real world, we have $\alpha_1=k_1\mathbf{1}, \alpha_2=k_2\mathbf{1}$, and $k_1>>k_2>0$.}

\REM{
We shall solved the above three problems theoretically in Section 5 and validated empirically in Section 6. For the rest of this paper, we will focus on the general solution of the three problems Eq. 2,3 and 4 above.}


\subsection{Prediction+Optimization}\label{sub:pred-opt}

For compactness we write Eq.\ref{eq:prob} as $\max_{x \in \mathcal{X}} f(x, \theta)$, where $f$ is the objective function and $\mathcal{X}$ is the feasible domain. The solver for $f$ is to solve $x^* = \text{argmax}_{x \in \mathcal{X}} f(x, \theta)$. With parameters $\theta$ known, Eq.\ref{eq:prob1}--\ref{eq:prob3} can be solved by mature solvers like Gurobi \cite{gurobi} and CPLEX \cite{cplex}. 

In prediction+optimization, $\theta$ is unknown and needs to be predicted from some observed features $\xi \in \Xi$. The prediction model is trained on a dataset $D = \{(\xi_i, \theta_i)\}_{i=1}^N$. In this paper, we consider the prediction model, namely $\Phi$, as a neural network parameterized with $\psi$. In traditional supervised learning, $\Phi_{\psi}$ is learned by minimizing a generic loss, e.g., L1 (Mean Absolute Error) or L2 (Mean Squared Error), which measures the expected distance between predictions and real values. However, such loss minimization is often inconsistent with the optimization objective $f$, especially when the prediction model is biased \cite{task-opt-nips2017}. 


\begin{figure}
    \centering
    \includegraphics[width=0.5\linewidth]{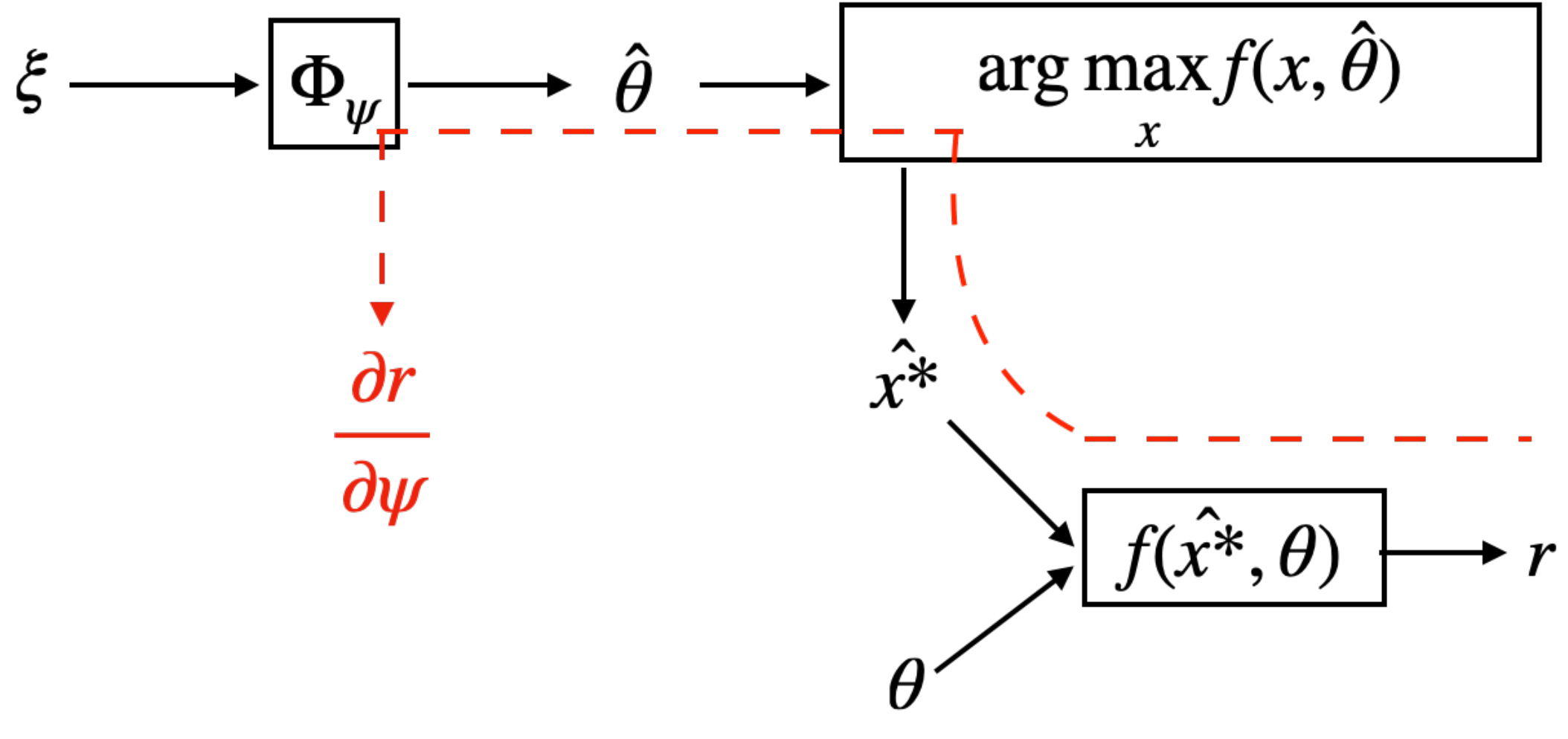}
    \caption{Computation graph of the decision-focused prediction methods.}
    \label{fig:computational_graph}
\end{figure}



Instead, decision-focused learning directly optimizes $\psi$ with respect to the optimization objective $f$, that is, $\max_{\psi} {\mathbf E}_{(\xi, \theta) \sim D}[f(\hat{x^*}(\Phi_{\psi}(\xi)), \theta)]$. The full computational flow is illustrated in Fig.\ref{fig:computational_graph}. In the gradient-based learning, update of $\psi$'s gradient is $\frac{\partial  r}{\partial \psi} = \frac{\partial \hat{\theta}}{\partial \psi} ~ \frac{\partial \hat{x^*}}{\partial \hat{\theta}} ~ \frac{\partial  r}{\partial \hat{x^*}}$, where utility $r=f(\hat{x^*},\theta)$. The Jacobian ${\partial \hat{\theta}}/{\partial \psi}$ is computed implicitly by auto-differentiation of deep learning frameworks (e.g, PyTorch \cite{pytorch}), and ${\partial  r}/{\partial \hat{x^*}}$ is analytical. The main challenge is to compute ${\partial \hat{x^*}}/{\partial \hat{\theta}}$, which depends on differentiating the $\argmax$ operation. One recent approach is to rewrite the objective to be convex (by adding a quadratic regularizer if necessary), build and differentiate the optimality conditions (e.g., KKT conditions) \cite{diffbook} which map $\hat{\theta}$ to the solution $\hat{x^*}$, and then apply implicit function theorem to obtain ${\partial \hat{x^*}}/{\partial \hat{\theta}}$. Alternatively, in this paper we propose a novel approach that rewrites the problem as an unconstrained problem with soft constraints and derives analytical solutions, based on our observations on the real-world problem structures and coefficient properties.





\section{Methodology}\label{sec:methodology}
Our main idea is to derive a surrogate function for $f$ with a closed-form solution such that the Jacobian $\frac{\partial x}{\partial \theta}$ is analytical, making the computation of gradient straightforward. Unlike other recent work \cite{optnet, amos2019differentiable, cvxpy, surrogate-melding}, our method does not need to solve KKT optimality condition system. Instead, by adding reasonable costs for infeasibility, we convert the constrained problem into an unconstrained one. With the assumption of concavity, we prove that there exist constant vectors $\beta_1, \beta_2, \beta_3$, such that Eq.\ref{eq:prob} can be equivalently transformed into an unconstrained problem:
\begin{equation}\label{eq:lagrangian}
\small
  \max_x L(x) = \max_x g(x, \theta) - \alpha^T \max(Cx-d, 0) -  \beta_1^T \max(Ax-b, 0) - \beta_2^T|Bx-c| - \beta_3^T \max(-x, 0)
 \end{equation}
The structure of this section is as follows. Section \ref{sub:hconstraints} proves that the three types of hard constraints can be softened by deriving bounds of $\beta_1, \beta_2, \beta_3$; for this paper, we will assign each entry of the three vectors the equal value (with a slight abuse of notation, we denote $\beta_1=\beta_2=\beta_3=\beta$ for the proofs of bounds; we align with the worst bound applicable to the problem formulation.) Section \ref{sub:sconstraints} proposes a novel surrogate function of $\max(\cdot, 0)$, such that the analytical form of $\frac{\partial x}{\partial{\theta}}$ can be easily derived via techniques of implicit differentiation \cite{diffbook} and matrix differential calculus \cite{matrixdiffcalculus} on equations derived by convexity \cite{cvxbook}. Based on such derived $\frac{\partial x}{\partial \theta}$, we develop our end-to-end learning algorithm of prediction+optimization whose detailed procedure is described in Appendix C. 




\subsection{Softening the Hard Constraints}\label{sub:hconstraints}
For any hard constraints $w=Ax-b \leq 0$, we denote its equivalent soft constraints as $H(w)=\beta^T \max(w, 0)$. $H(w)$ should satisfy two conditions: 1) for $w \leq 0$ (feasible $x$), $H(w)=0$; 2) for $w\geq 0$ (infeasible $x$), $H(w)$ is larger than \textit{the utility gain}(i.e., improvement of the objective value) $R=f(x,\theta)-\max_{x_1:Ax_1\leq b}f(x_1,\theta)$ by violating $Ax-b\leq 0$. Intuitively, the second condition requires a sufficiently large-valued $\beta>0$ to ensure that the optimization on the unconstrained surrogate objective never violates the original $Ax\leq b$; to make this possible, we assume that the $l2$-norm of the derivative of the objective $f$ before conversion is bounded by constant $E$. The difficulty of requirement 2) is that the distance of a point to the convex hull $l$ is not bounded by the sum of distances between the point and each hyper-plane in general cases, so the utility gain obtained from violating constraints is unbounded. Fig.~\ref{fig:fig1}-(a) shows such an example which features the small angle between hyper-planes of the activated cone. We will refer such kind of 'unbounding' as "\textit{acute angles}" below.


\begin{figure}
\centering
\subfigure[2D acute angle]{
\begin{minipage}[b]{0.25\linewidth}
\includegraphics[width=0.88\linewidth]{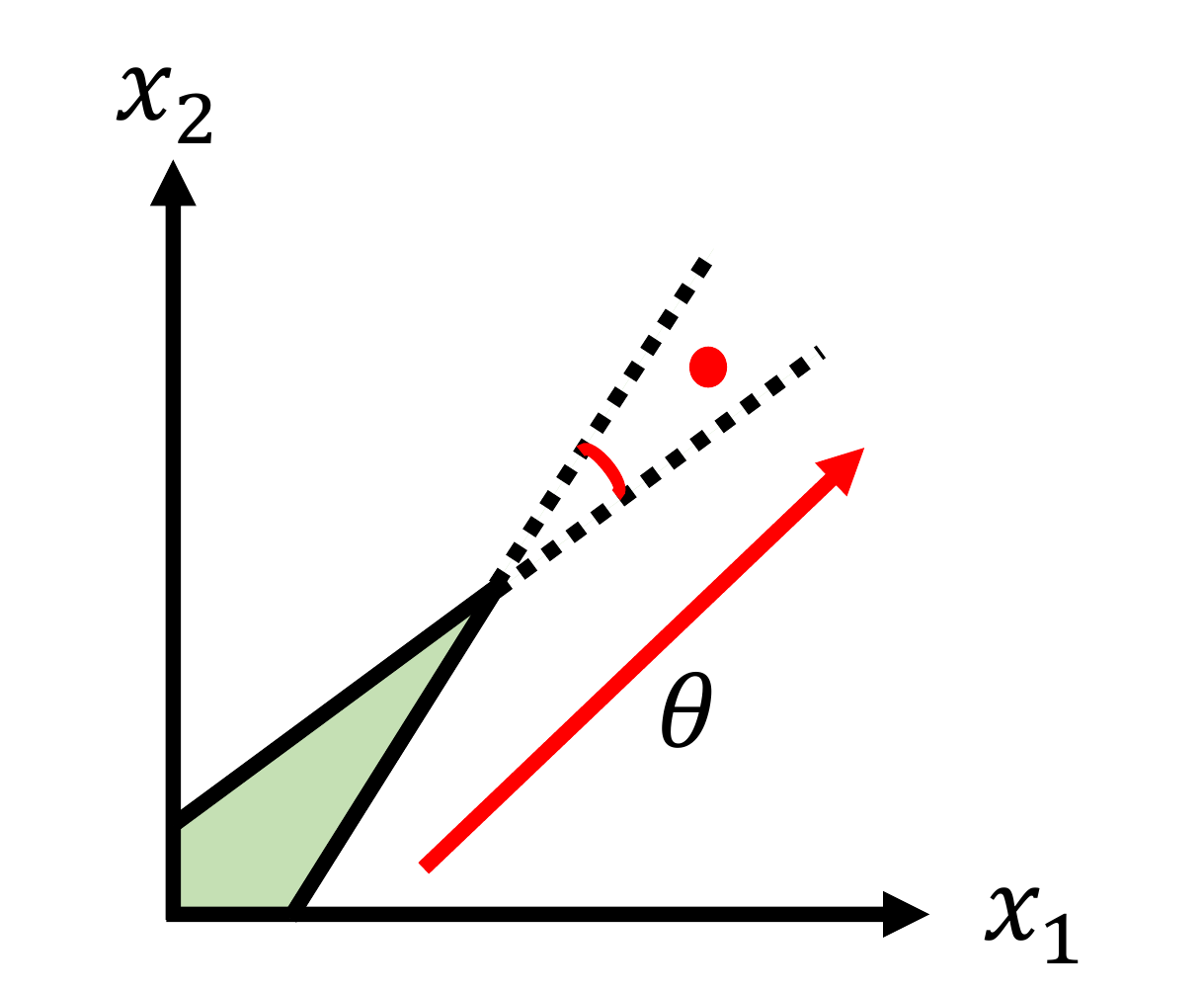}
\end{minipage}}
\subfigure[3D acute angle]{
\begin{minipage}[b]{0.23\linewidth}
\centering
\includegraphics[height=2.4cm]{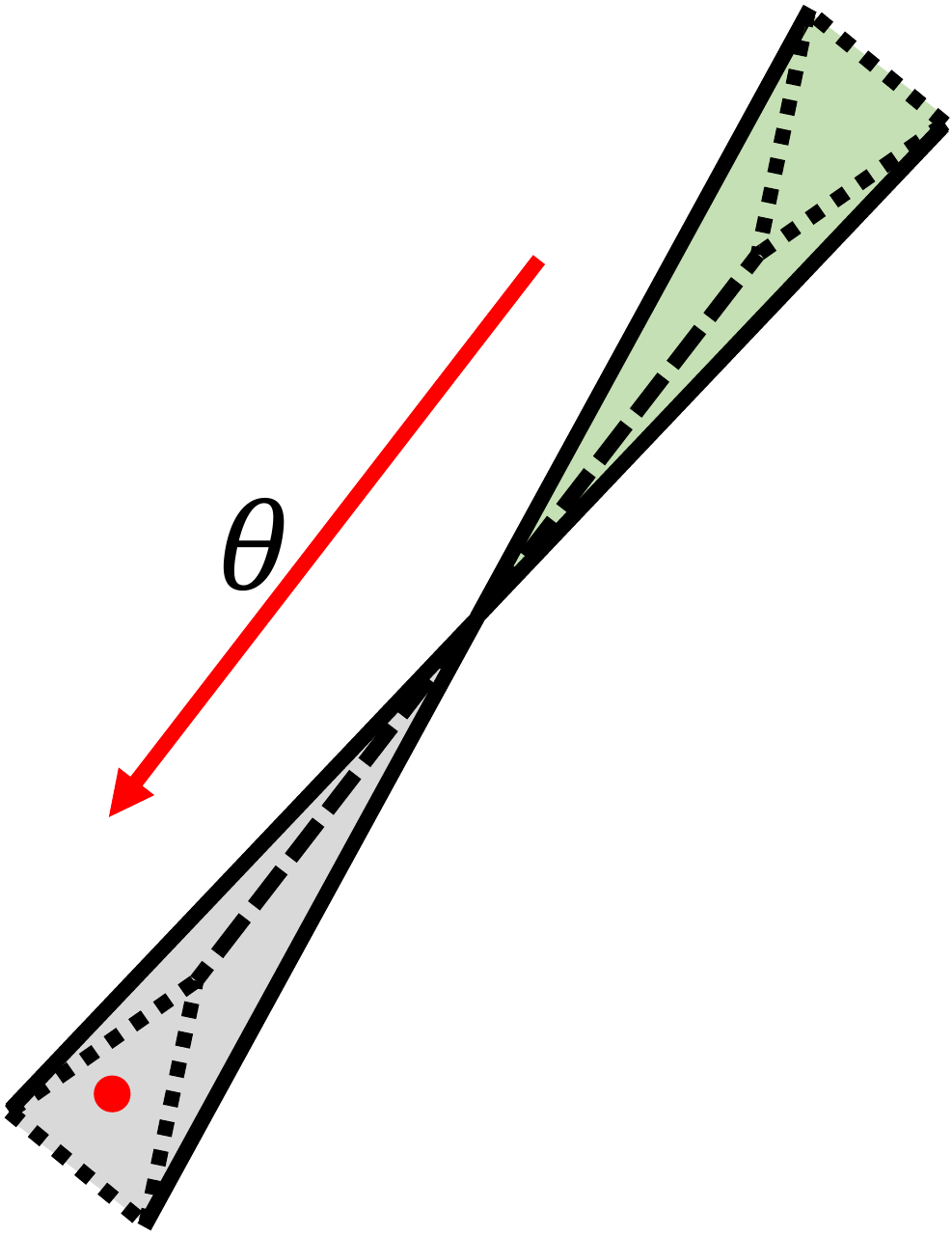}
\end{minipage}}
\subfigure[Non-degraded cone]{
\begin{minipage}[b]{0.23\linewidth}
\centering
\includegraphics[height=2.4cm]{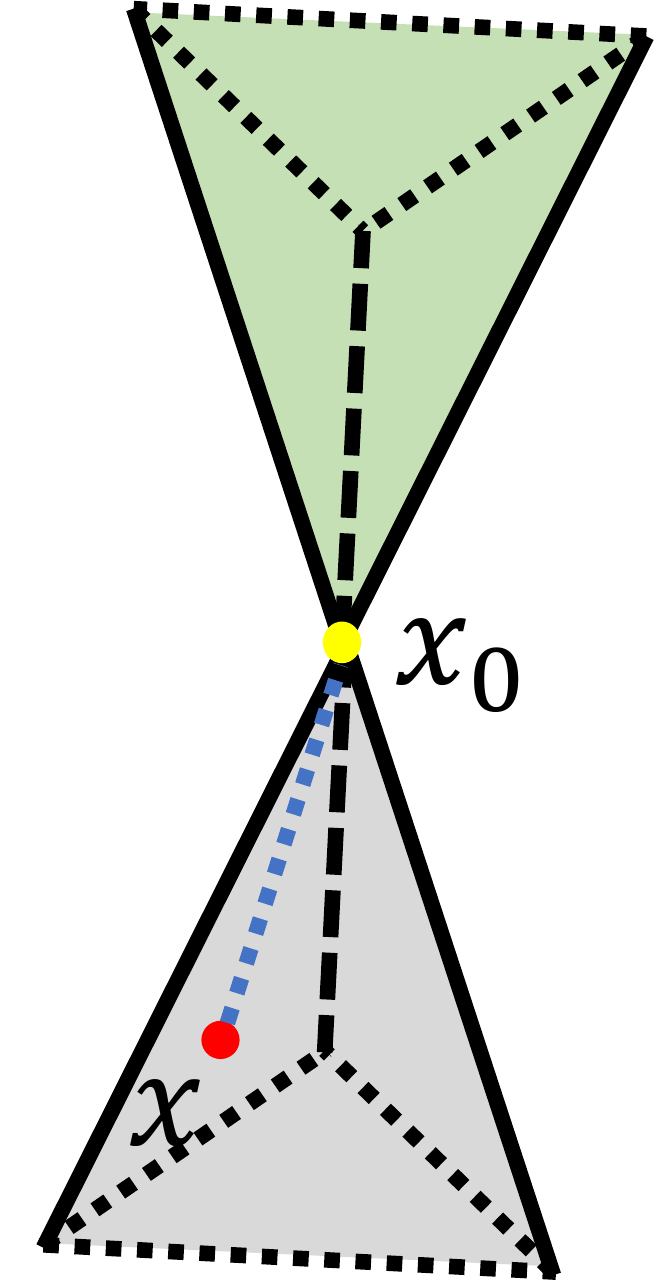}\end{minipage}}
\subfigure[Degraded cone]{
\begin{minipage}[b]{0.23\linewidth}
\centering
\includegraphics[height=2.4cm]{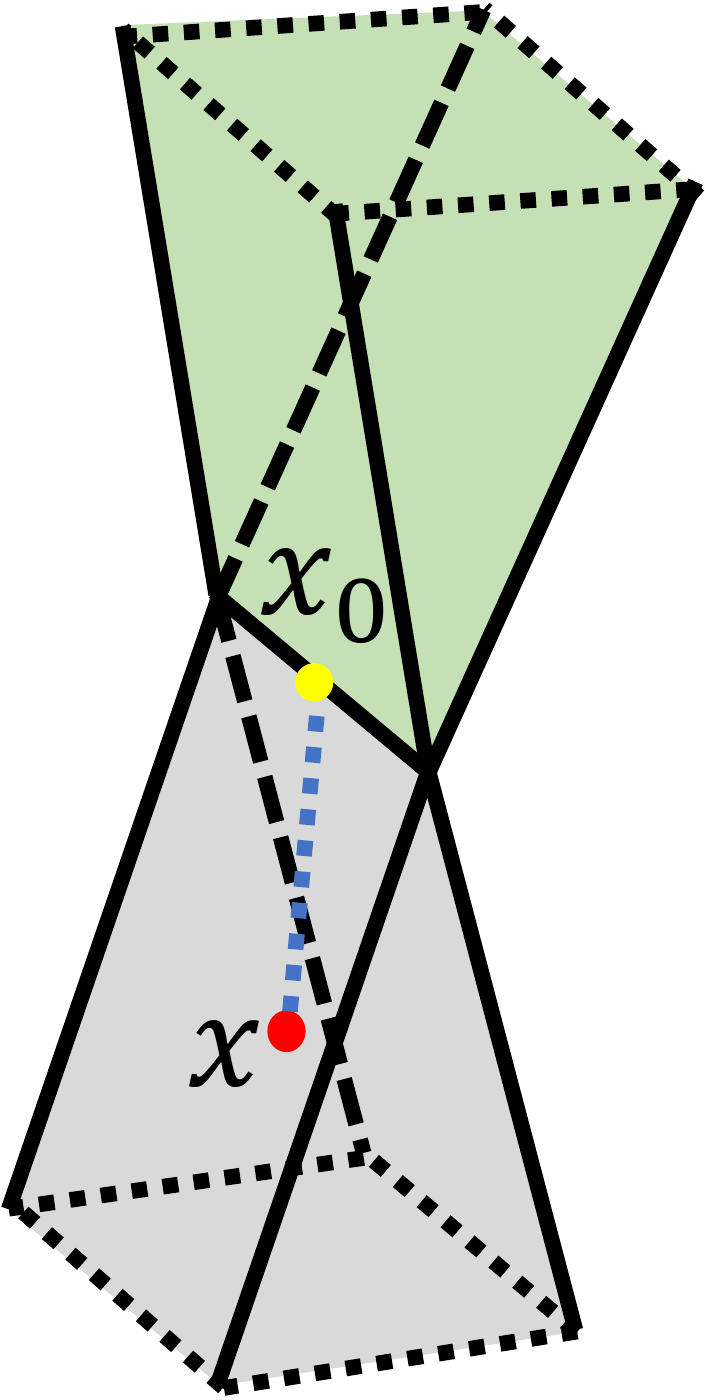}
\end{minipage}}
\caption{(a) and (b) are $2$ and $3$-dimensional "\textit{acute angles}"; (c) and (d) shows two corresponding \textit{activated cones} for given acute angles. The green area is the feasible region, $x$ is the red point and $x_0$ is the yellow point; the red $\theta$ is the derivative of an objective $g(x,\theta)=\theta^Tx$.} 
\label{fig:fig1}
\end{figure}

The main effort of this subsection is to analyze and bound the effect caused by such acute angles. Given a convex hull $\mathcal{C}=\{z \in \mathbb{R}^n | Az \leq b\}$ ($A\geq 0$ is not required here) and any point $x\not\in \mathcal{C}$, let $x_0\in \mathcal{C}$ be the nearest point of $\mathcal{C}$ to $x$, and $A'x\geq b'$ represent all active constraints at $x$, then all such active constraints must pass through $x_0$.\footnote{Otherwise, there must exist an active constraint $i$ and a  small neighbourhood of $x_0$, say $B(x_0, \epsilon)$, such that $\forall z\in B(x_0, \epsilon)$, $A_iz-b_i\neq 0$, which implies that either constraint $i$ is inactive ($<0$) or $x_0 \not\in \mathcal{C}$ ($>0$).} The domain $\mathcal{K} = \{z \in \mathbb{R}^n | A'z \geq b'\}$ is a cone or degraded cone where the tip of the cone is a subspace of $R^n$. For the rest of the paper, we will call $\mathcal{K}$ \textit{activated cone}, as shown in Fig.~\ref{fig:fig1}. Note that for any degraded activated cone, $x-x_0$ is always perpendicular to the tip subspace; therefore, we may project the cone onto the complementary space of the tip and get the same multiplier bound on the projected activated cone with lower dimensions.  


Ideally, we aim to eliminate the utility gain $R$ obtained from violating $A'x\leq b'$ with the penalty $\beta^T(A'x-b')$, \ie{} ensure $\beta^T(A'x-b')\geq R$ hold for any $x$. For the compactness of symbols and simplicity, we will assume that the main objective is $g(x,\theta)=\theta^Tx$ in this section; however note that our proofs apply with the existence of soft constraints and quadratic terms in $g$, which is discussed at the beginning of Appendix A. With such assumption, we now give a crucial lemma, which is the core of our proof for most of our theorem:
\begin{lemma}
(Bounding the utility gain) Let $R=f(x,\theta)-\max_{x_1\in\mathcal{C}}f(x_1,\theta)$ be the utility gain, then $R\leq f(x,\theta)-f(x_0, \theta)\leq \frac{E}{\cos p_0}\sum_{i=1}^{n}(A'_ix-b'_i)$, where $x$ is an infeasible point, $A'x \leq b'$ the active constraints at $x$, $p_0 = \angle({A'_i}^*, \theta^*)$ where ${A'_i}^*$ and $\theta^*$ are the optimal solution of $\max_{\theta}\min_{A'_i} cos \angle(A'_{i}, \theta)$ (i.e., the maximin angle between $\theta$ and any hyperplane of the activated cone $\mathcal{K} = \{z \in \mathbb{R}^n | z-x_0\in cone(A')\}$), and $x_0$ the projection of $x$ to the tip of cone $\mathcal{C} = \{ z \in \mathbb{R}^n | A'z \leq b'\}$. $E$ is the upper bound of $||\theta||_2$. $\angle(\cdot, \cdot)$ denotes the angle of two vectors. 
\end{lemma}
Thus, $\frac{E}{\cos p_0}\mathbf{1}$ is a feasible choice of $\beta$, and it suffices by finding the lower bound for $\cos p_0$. For the rest of the section, we find the lower bound of $\cos p_0$ by exploiting the assumed properties of $A'$, \eg{} $A'\geq 0$; we give the explanation of the full proofs for all theoretical results in Appendix A.

\subsubsection{Conversion of Inequality Constraints $Ax \leq b$}
Let us first consider the constraints $w = Ax-b \leq 0$, where $A\geq 0$, $b\geq 0$. It is easy to prove that given $A\geq 0$ and $b\geq 0$, the distance of a point to the convex hull $Ax \leq b$ is bounded. More rigorously, we have the following theorem, which guarantees the feasibility of softening the hard constraints of inequalities $Ax\leq b$: 


\REM{ 
\begin{lemma}
For two $(n-1)$-dimensional hyper-planes $\alpha_1^Tx=b_1$, $\alpha_2^Tx=b_2$ in $\mathbb{R}^n$, if $\alpha_1,\alpha_2,b_1,b_2\geq 0$, then their inner product of their outward normal vectors (assume the first non-zero entry is positive) $n_1^Tn_2\geq 0$. 
\end{lemma}
}

\begin{theorem}
Assume the optimization objective $\theta^Tx$ with constraints $Ax\leq b$, where $A\geq 0$, and $b\geq 0$. Then, the utility gain $R$ obtained from violating $Ax-b \leq 0$ has an upper bound of $O(\sum_i\max(w_i, 0)E)$, where $w=A'x-b'$, and $A'x \leq b'$ is the active constraints.
\end{theorem}

\subsubsection{Conversion of Inequality Constraints $x \geq 0$}

\par With inequality constraints $Ax\leq b$ converted, we now further enforce $x\geq 0$ into soft constraints. It seems that the constraint of this type may form a very small angle to those in $w=Ax-b$. However, as $-x$ is aligned with the axes, we can augment $x\geq 0$ into soft constraints by proving the following theorem:
\begin{theorem}
When there is at least one entry of $x\geq 0$ in the activated cone, the utility gain $R$ from deviating the feasible region is bounded by $O(\frac{n^{1.5}E\sum_{i}max(w_i, 0)}{\sin{p}})$, where $p$ is the smallest angle between axes and other constraint hyper-planes and $w$ is the union of $Ax-b$ and $-x$. 
\end{theorem}
Hence, we can set $\beta=O(\frac{n^{1.5}E}{\sin p})$. Specially, for binary constraints we may set $\beta=O(n^{1.5}E)$:
\begin{corollary}
For binary constraints where the entries of $A$ are either $0$ or $1$, the utility gain $R$ of violating $x\geq 0$ constraint is bounded by $O(n^{1.5}E\sum_{i}\max(w_i, 0))$, where $w_i=A_ix-b_i$ or $-x$.
\end{corollary}
which gives a better bound for a set of unweighted item selection problem (e.g. select at most $k$ items from a particular subset).


\subsubsection{Conversion of Equality Constraints $Bx = c$}
Finally, we convert $Bx=c$ into soft constraints. This is particularly difficult, as $Bx=c$ implies both $Bx\leq c$ and $-Bx\leq -c$, which will almost always cause acute angles. Let's first consider a special case where there is only one equality constraints and $A$ is an element matrix $I^{n \times n}$.



\begin{theorem}\label{simle_theorem}
If there is only one equality constraint $Bx=c$ (i.e., $B$ degrades as a row vector, such like $\sum_{i}x_i=1$) and special inequality constraints $x\geq 0$, $Ix \leq b$, 
then the utility gain $R$ from violating constraints is bounded by $O(\frac{n^{1.5}E\sum_{i}\max(w_i, 0)}{\sin{p}})$, where $p$ is the same with theorem 3, $w$ is the union of $Bx-c$ and $-x$.
\end{theorem}
Intuitively, when there is only one equality constraint, the equality constraint can be viewed as an unequal one, for at most one side of the constraint can be in a non-degraded activated cone. Thus, we can directly apply the proof of Theorem 2 and 3, deriving the same bound $O(\frac{n^{1.5}E}{\sin p})$ for $\beta$. 

Finally, we give bounds for general $Ax\leq b$, $Bx=c$ with $A, b, B, c\geq 0$ as below:
\begin{theorem}
Given constraints $Ax\leq b$, $x\geq 0$, and $Bx=c$, where $A, B, b, c\geq 0$, the utility gain $R$ obtained from violating constraints is bounded by $O(\sqrt{n}\lambda_{max}\sum_{i}\max(w_i, 0))$, where $\lambda_{max}$ is the upper bound for eigenvalues of $P^TP$ ($P: x\rightarrow Px$ is an orthogonal transformation for an $n$-sized subset of normalized row vectors in $A, B$ and $-I$), and $w$ is the union of all active constraints from $Ax\leq b$, $x\geq 0$, $Bx\leq c$ and $-Bx\leq -c$.
\end{theorem}

In this theorem, $P$ is generated by taking an arbitrary $n$-sized subset from the union of row vectors in $A, B$ and $-I$, orthogonizing the subset, and using the orthogonal transformation matrix as $P$; there are $\binom{n+m_1+m_2}{n}$ different cases of $P$, and $\lambda_{max}$ is the upper bound of eigenvalues of $P^TP$ over all possible cases of $P$. Note that there are no direct bounds on $\lambda_{max}$ with respect to $n$ and the angles between hyper-planes. However, empirical results (see Appendix A for details) show that for a diverse set of synthetic data distributions, $\lambda_{max}=O(n^2)$ follows. Therefore, empirically we can use a bound $O(\frac{n^{2.5}E}{\sin p})$ for $\beta$.\footnote{For real optimization problems, we can sample $A'$ from $A$ and estimate the largest eigenvalue of $P^TP$.} 
So far, we have proven that all hard constraints can be transformed into soft constraints with bounded multipliers. For compactness, Eq.\ref{eq:lagrangian} is rewritten in a unified form:
\begin{equation} \label{eq:lagrangian1}
\small
\begin{aligned}
L(x) = g(x, \theta) - \gamma^T \cdot max(C' x - d'), \text{ where } 
~\gamma = {\begin{bmatrix}\alpha\\ O(\frac{n^{2.5}E}{\sin{p}})\mathbf{1} \\ \end{bmatrix}}, 
~C' = {\scriptsize \begin{bmatrix}C \\ A \\-B\\ B \\ -I\end{bmatrix}}, 
~d'={\scriptsize\begin{bmatrix} d \\ b \\-c\\ c \\ 0 \end{bmatrix}}
\end{aligned}
\end{equation}




\subsection{The Unconstrained Soft Constraints}\label{sub:sconstraints}

As Eq.\ref{eq:lagrangian} is non-differentiable for the $\max$ operator, we need to seek a relaxing surrogate for differentiation. The most apparent choice of the \textit{derivative} of such surrogate $S(z)$ for $z=C'x-d'$ is sigmoidal functions; however, it is difficult to derive a closed-form solution for such functions, since sigmoidal functions yield $z$ in the denominator, and $z$ cannot be directly solved because $C$ is not invertible (referring to Appendix B for detailed reasons). Therefore, we have to consider a piecewise roundabout where we can first numerically solve the optimal point to determine which segment the optimal point is on, and then expand the segment to the whole space. To make this feasible, two assumptions must be made: 1) this function must be differentiable, and 2) the optimal point must be unique; to ensure this, the surrogate should be a convex/concave piece-wise function. The second property is for maintaining the optimal point upon segment expansion. Fortunately, there is one simple surrogate function satisfying our requirement:
\begin{equation}
\small
S(z)=\begin{cases}
          0 \quad &\text{if} \, z<-\frac{1}{4K} \\
          K(z+\frac{1}{4K})^2 \quad &\text{if} \, -\frac{1}{4K}\leq z\leq\frac{1}{4K} \\
          z \quad &\text{if} \, z\geq\frac{1}{4K} \\
     \end{cases}
\end{equation}
Let $M$ and $U$ be diagonal matrices as the indicator of $S(z)$, which are $\small M_{i,i}=2K[-\frac{1}{4K}\leq z_i\leq \frac{1}{4K}]$ and $U_{i,i}=[\frac{1}{4K}<z_i]$, where $[\cdot]$ is an indicator function.
$K>0$ is a hyper-parameter that needs to be balanced. Larger $K$ makes the function closer to original; however, if $K$ is too large, then the training process would be destabilized, because when the prediction error is large at the start of the training process, $\frac{\partial f}{\partial\hat{\theta}}|_{\hat{x}}$ might be too steep. Then consider the optimal point for the unconstrained optimization problem maximizing $\theta^Tx-\gamma^T\max(C'x-d', 0)$, by differentiaing on both sides, we can obtain:
\begin{equation}
\label{eq:solution1}
\small 
\theta=C'^TM \text{diag}(\gamma)(C'x-d')+C'^T(\frac{1}{4K}M+U)\gamma
\end{equation}

This equation reveals the closed-form solution of $x$ with respect to $\theta$, and thus the calculation of $\frac{\partial x}{\partial \theta}$ becomes straightforward:

\begin{equation}
\frac{\partial x}{\partial \theta}=(C'^TM \text{diag}(\gamma)C')^{-1}    
\end{equation}

$C'^TM \text{diag}(\gamma)C'$ is invertible as long as at least $n$ soft constraints are on the quadratic segment (\ie{} active), which is the necessary condition to fix a particular optimal point $n$ in $\mathcal{R}^n$.  With such solution, we can train our prediction model with stochastic gradient descent (SGD). For each batch of problem instances, we first solve optimization problems numerically using solvers like Gurobi to get the matrix $M$ and $U$, and then calculate gradients with the analytical solution. The parameters of the prediction model are updated by such gradients. The sketch of our algorithm is outlined in Appendix C.

\REM{
\section{Applications}\label{sec:apps}
In this section, we present the analytical derivatives of the three applications whose formulations were given in Eq.\ref{eq:prob1}, \ref{eq:prob2} and \ref{eq:prob3} respectively. Detailed derivation are given in appendix. Note that for symbol compactness, through this section we will reuse $C, d, \alpha$ to describe the new soft constraints, and denote their original variables in Eq.\ref{eq:prob1}$\sim$\ref{eq:prob3} as $C_0$, $d_0$ and $\alpha_0$.

In this section, we will illustrate that our framework is suitable for solving a wide range of problems by solving the three problems Eq. 2, 3 and 4 proposed in the preliminaries. These problems range from linear to quadratic, and the prediction variable can be in either the objective or the soft constraints. Only the results and necessary denotations are given in this section; see appendices for a detailed derivation.
}


\section{Applications and Experiments}\label{sec:experiments}

We apply and evaluate our approach on the three problems described in Section 2, \ie{linear programming, quadratic programming, and asymmetric soft constraint minimization}. These problems are closely related to three applications respectively: synthetic linear programming, portfolio optimization, and resource provisioning, which are constructed using synthetic or real-world datasets. The detailed derivation of gradients for each application can be found in Appendix D.

In our experiments, the performance is measured in regret, which is the difference between the objective value when solving optimization over predicted parameters and the objective value when solving optimization over actual parameters. For each experiment, we choose two generic two-stage methods with $L1$-loss and $L2$-loss, as well as decision-focused methods for comparison baselines. We choose both SPO+\cite{elmachtoub2020smart} and DF proposed by Wilder \etaldot{} \cite{aaai-WilderDT19} for synthetic linear programming and DF only for portfolio optimization,\footnote{The method \cite{surrogate-melding} and its variant with dimension reduction \cite{aaai-WilderDT19} have same performance in terms of regret on this problem, thus we choose the former for comparison convenience.} as the former is specially designed for for linear objective. For resource provisioning, we use a problem-specific weighted L1 loss, as both SPO+ and DF are not designed for gradients with respect to variables in the soft constraints.
All reported results for each method are obtained by averaging on $15$ independent runs with different random seeds.

As real-world data is more lenient than extreme cases, in practice we use a much lower empirical bound than the upper bound proved in section 3.1., e.g., constants of around $20$ and $5\sqrt{n}$ where $n$ is the number of dimensions of decision variables. One rule of thumb is to start from a reasonable constant or a constant times $\sqrt{n}$, where such "reasonable constant" is the product of a constant factor (e.g. $5-10$) and a roughly estimated upper bound of $||\theta||_2$ (which corresponds to $E$ in our bounds) with specific problem settings; then alternately increase the constant and time an extra $\sqrt{n}$ while resetting the constant until the program stops diverging, and the hard constraints are satisfied. In our experiments, we hardly observe situations where such process goes for two or more steps.




\subsection{Synthetic Linear Programming}


\textbf{Problem setup.} The prediction dataset $\{(\xi_i, \theta_i)\}_{i=1}^{N}$ is generated by a general structural causal model (\cite{peters2017elements}), ensuring it is representative and consistent with physical process in nature. The programming parameters are generated for various scales in numbers of decision variables, hard constraints, and soft constraints. Full details are given in Appendix E. 

\textbf{Experimental setup.} All five methods use the same prediction model -- a fully connected neural network of two hidden layers with $128$ neurons for each and ReLU \cite{relu} for activation. We use AdaGrad \cite{adagrad} as the optimizer, with learning rate $0.01$ and gradient clipped at $1e-4$. We train each method for $40$ epochs, and early stop when valid performance degrades for $4$ consecutive epochs. Specially, to make DF\cite{aaai-WilderDT19} work right on the non-differentiable soft constraints, we first use black-box solvers to determine whether each soft constraint is active on the optimal point, and then optimize with its local expression (i.e. 2nd-order Taylor expansion at optimal point). 

\textbf{Performance analysis.} Our experiments cover four programming problem scales with three prediction dataset sizes. Results are summarized in Table \ref{tab:exp1}. In all cases, our method performs consistently better than two-stage methods, DF and SPO+. Even for the cases with only hard constraints (\ie{} the third parameter of problem size is 0), our method still has significant advantage, demonstrating its effectiveness on handling hard constraints. Surprisingly, although the main objective is linear, SPO+ often performs even worse than two-stage methods. Detailed analysis (see the appendix)
shows that SPO+ quickly reaches the test optimality and then over-fits. This may be due to that, unlike our method, SPO+ loss is not designed to align with the soft constraints. This unawareness of soft constraint is also why DF is performing worse than our method, as DF is working on an optimization landscape that is non-differentiable at the boundary of soft constraints, on which the optimal point usually lies. Besides, with the increment of the samples in train data, the performance of all methods is improved significantly and the performance gap among ours and two-stage methods becomes narrow, which implies that prediction of two-stage methods becomes better and with lower biases. Even so, our method has better sample efficiency than two-stage methods. 

We also investigated effect of the hyper-parameter $K$ in our surrogate $\max$ function, detailed in the appendix. 
Through our experiments, $K$'s effect to regret  is not monotonic, and its optimal value varies for different problem sizes. Interestingly, $K$'s effect is approximately smooth. Thus, in practice, we use simple grid search to efficiently find the best setting of $K$.

\begin{table}[t]
    \centering
    \begin{tabular}{c|c|c|c|c|c|c}
    & & \multicolumn{5}{c}{Regret (the lower, the better)}\\
    \hline 
    $N$ & Problem Size & L1 & L2 & SPO+~\cite{elmachtoub2020smart} & DF~\cite{aaai-WilderDT19} & Ours\\
    \hline
        100 & (40, 40, 0) & 2.454$\pm$0.232 & 2.493$\pm$0.295 & 2.506$\pm$0.294 & 2.478$\pm$0.425 & \textbf{2.258$\pm$0.311}\\
            & (40, 40, 20) & 2.626$\pm$0.307 & 2.664$\pm$0.303 & 2.667$\pm$0.281 & 2.536$\pm$0.376 & \textbf{2.350$\pm$0.263} \\
            & (80, 80, 0) & 5.736$\pm$0.291 & 5.831$\pm$0.361 & 5.711$\pm$0.309 & 5.756$\pm$0.317 & \textbf{5.200$\pm$0.506} \\
            & (80, 80, 40) & 4.786$\pm$0.403 & 4.786$\pm$0.596 & 4.939$\pm$0.382 & 4.902$\pm$0.537 & \textbf{4.570$\pm$0.390} \\
    \hline 
       1000 & (40, 40, 0) & 1.463$\pm$0.143 & 1.447$\pm$0.155 & 1.454$\pm$0.148 & 1.434$\pm$0.268 & \textbf{1.346$\pm$0.144} \\
            & (40, 40, 20) & 1.626$\pm$0.141  & 1.613$\pm$0.110  & 1.618$\pm$0.103 & 1.529$\pm$0.151 &\textbf{1.506$\pm$0.102}
            \\
            & (80, 80, 0) & 3.768$\pm$0.132 & 3.718$\pm$0.117 & 3.573$\pm$0.113 & 3.532$\pm$0.102 &\textbf{3.431$\pm$0.100} \\
            & (80, 80, 40) & 2.982$\pm$0.176 & 2.913$\pm$0.172 & 2.879$\pm$0.148 & 3.351$\pm$0.212 &\textbf{2.781$\pm$0.165} \\
    \hline
        5000 & (40, 40, 0) & 1.077$\pm$0.105 & 1.080$\pm$0.109 & 1.090$\pm$0.105 & 1.078$\pm$0.092 &\textbf{1.037$\pm$0.100}
        \\
            & (40, 40, 20) & 1.283$\pm$0.070 & 1.277$\pm$0.077 & 1.298$\pm$0.077 & 1.291$\pm$0.091 &\textbf{1.220$\pm$0.071} \\
            & (80, 80, 0) & 2.959$\pm$0.086 & 2.943$\pm$0.091 & 2.926$\pm$0.079 & 2.869$\pm$0.085 &\textbf{2.845$\pm$0.064} \\
            & (80, 80, 40) & 2.239$\pm$0.122 & 2.224$\pm$0.106 & 2.234$\pm$0.122 & 2.748$\pm$0.165 &\textbf{2.172$\pm$0.098} \\
    \hline
    \end{tabular}
    \caption{Performance comparison (regret mean with std. deviation) for the synthetic linear programming problem. $N$ is the size of the training dataset, and problem size is a triplet (\# of decision variables' dimension, \# of hard constraints, \# of soft constraints).}
    \label{tab:exp1}
\end{table}



\subsection{Portfolio Optimization}

\textbf{Problem and experimental setup.} 
\REM{As described in Eq.\ref{eq:prob2} and Section \ref{sub:app2}, we extend the classic minimum variance portfolio with a term of soft constraints.}
The prediction dataset is daily price data of SP500 from 2004 to 2017 downloaded by Quandl API \cite{sp500-dataset} with the same settings in \cite{surrogate-melding}. We use the same fix as that in linear programming experiment to make DF\cite{aaai-WilderDT19} work with non-differentiable soft constraints, which was also used in \cite{surrogate-melding} for non-convex optimization applications. Most settings are aligned with those in \cite{surrogate-melding}, including dataset configuration, prediction model, learning rate ($0.01$), optimizer (Adam), gradient clip ($0.01$) and number of training epochs ($20$). We set the number of soft constraints to $0.4$ times of $n$, where $n$ is the number of candidate equities. For the soft constraint $\alpha^T\max(Cx-d, 0)$, $\alpha=\frac{15}{n}v$, where each element of $v$ is generated randomly at uniform from $(0, 1)$; the elements of matrix $C$ is generated independently from $\{0, 1\}$, where the probability of $0$ is $0.9$ and $1$ is $0.1$. $K$ is set as $100$.

\textbf{Performance analysis.} Table \ref{tab:perf2} summarizes the experimental results. In total, on all problem sizes (\#equities), our method performs consistently and significantly better than both two-stage (L1 and L2) methods and the decision focused DF\cite{aaai-WilderDT19}. Among the three baselines, DF is significantly better than two-stage methods, similar to results in \cite{surrogate-melding}. In fact, DF under this setting can be viewed as a variant of our method with infinite $K$ and no conversion of softening $\sum_i x_i=1$. The comparison to DF also demonstrates the advantage of our method on processing such non-differentiable cases against the simple 2nd-order Taylor expansion. Besides, with the increment of the number of equities, regrets of all methods decease, which indicates that for the constraint $\sum_i x_i=1$, larger number of equities brings smaller entries of $x$ on average (with the presence of $Q$, there are many non-zero entries of $x$), lowering the impact of prediction error for any single entry.
\REM{
in general, the mean and std. deviation decreases as $N$ grows for all methods, because with the constraint $\sum_i x_i=1$, larger $N$ brings smaller entries of $x$ on average (with the presence of $Q$, the number of non-zero entries of $x$ are many), lowering the impact of prediction error for any single entry. For all problem scale following Wilder et al. \cite{surrogate-melding}'s work, our method is significantly better than both decision-focused (DF) and two-stage(TS) methods with lower regret and smaller std. deviation. Among the three baselines, DF is significantly better than TS, and $L1$-loss has similar performance to $L2$-loss, which is similar to Wilder's work. In fact, DF under this setting can be viewed as a variant of our method with $K$ goes to infinity and without conversion of hard constraint $\sum_i x_i=1$; the result shows that our method performs better in the non-differentiable environment than simple 2nd-order Taylor expansion.}
\begin{table}[t]
    \centering
    \begin{tabular}{c|c|c|c|c}
          & \multicolumn{4}{c}{Regret measured in \% (the lower, the better)}\\
          \hline
         \#Equities & L1 & L2 & DF~\cite{aaai-WilderDT19} & ours($K=100$) \\
         \hline
        50 & 4.426$\pm$0.386 & 4.472$\pm$0.385 & 4.016$\pm$0.389 & \textbf{3.662$\pm$0.238} \\
        100 & 4.262$\pm$0.231 & 4.320$\pm$0.229 & 3.500$\pm$0.252 & \textbf{3.214$\pm$0.138} \\
        150 & 3.878$\pm$0.281 & 3.950$\pm$0.287 & 3.419$\pm$0.281 & \textbf{3.109$\pm$0.162}\\
        200 & 3.755$\pm$0.236 & 3.822$\pm$0.273 & 3.406$\pm$0.287 & \textbf{3.152$\pm$0.183}\\
        250 & 3.721$\pm$0.205 & 3.751$\pm$0.212 & 3.335$\pm$0.175 & \textbf{3.212$\pm$0.135}\\
    \hline
    \end{tabular}
    \caption{Performance comparison (regret mean with std. deviation) for portfolio optimization.}
    \label{tab:perf2}
\end{table}

\subsection{Resource Provisioning} 
\textbf{Problem setup.} We use ERCOT energy dataset \cite{ercot-dataset}, which contains hourly data of energy output from 2013 to 2018, $52535$ data points in total. We use the last $20\%$ samples for test. We aim to predict the matrix $C\in\mathbb{R}^{24\times 8}$, the loads of $24$ hours in $8$ regions. The decision vairable $x$ is $8$-dimensional, and $d=0.5\times\mathbf{1}+0.1N(0, 1)$. We test five sets of $(\alpha_1, \alpha_2)$, with $\alpha_1/\alpha_2$ ranging from $100$ to $0.01$.

\textbf{Experimental setup.} We use AdaGrad with learning rate of $0.01$, and clip the gradient with norm $0.01$. For the prediction model, we use McElWee's network \cite{kevein-blog} which was highly optimized for this task, with $(8\times 24\times 77)$-dimensional numerical features and embedding ones as input. 

\textbf{Performance analysis.} Table \ref{tab:perf3} shows the experimental results. The absolute value of regret differs largely across different ratios of $\alpha_1/\alpha_2$.
Our method is better than other methods, except for $\alpha_1/\alpha_2=1$, where the desired objective is exactly $L1$-loss and thus $L1$ performs the best. Interestingly, compared to L1/L2, the Weighted L1 performs better when $\alpha_1/\alpha_2>1$, but relatively worse otherwise. This is probably due to the dataset's inherent sample bias (\eg{} asymmetric distribution and occasional peaks), which causes the systematic bias (usually underestimation) of the prediction model. This bias exacerbates, when superposed with weighted penalty multipliers which encourage the existing bias to fall on the wrong side. Besides, the large variance for weighted L1 at $\alpha_1/\alpha_2=0.01$ is caused by a pair of outliers.
\begin{table}[t]
    \centering
    \begin{tabular}{c|c|c|c|c}
        & \multicolumn{4}{c}{Regret (the lower, the better)}\\
        \hline
        $\alpha_1/\alpha_2$ & L1 & L2 & Weighted L1  & Ours($K=0.05$) \\
        \hline
        100 & 105.061$\pm$21.954 & 93.193$\pm$29.815 & 79.014$\pm$32.069 & \textbf{20.829$\pm$8.289}\\
        10 & 13.061$\pm$2.713 & 13.275$\pm$6.208 & 7.743$\pm$1.305 & \textbf{2.746$\pm$1.296}\\
        1 & \textbf{4.267$\pm$0.618} & 5.136$\pm$0.722 & \textbf{4.267$\pm$0.618} & 5.839$\pm$0.512\\
        0.1 & 10.846$\pm$1.606& 13.619$\pm$2.195 & 16.462$\pm$2.093 & \textbf{10.240$\pm$1.248} \\
        0.01 & 99.145$\pm$21.159& 118.112$\pm$29.957 & 230.825$\pm$91.184 & \textbf{94.341$\pm$29.821}\\
    \hline
    \end{tabular}
    \caption{Performance comparison (regret mean with std. deviation) for resource provisioning.}
    \label{tab:perf3}
\end{table}

\section{Related Work} 

\textbf{Differentiating $\bm{\argmin}$/$\bm{\argmax}$ through optimality conditions.} For convex optimization problems, the KKT conditions map coefficients to the set of solutions, and thus can be differentiated for $\argmin$ using implicit function theorem. Following this idea, existing work 
developed implicit layers of argmin in neural network, including OptNet \cite{optnet} for quadratic programs (QP) problems and CVXPY \cite{amos2019differentiable} for more general convex optimization problems. Further with linear relaxation and QP regularization, Wilder \etaldot{} derived an end-to-end framework for combinatorial programs \cite{aaai-WilderDT19}, which accelerates the computation by leverage the low-rank properties of decision vectors \cite{surrogate-melding}, and is further extended to mixed integer linear programs in MIPaaL \cite{mipaal}. Besides, for the relaxed LP problems, instead of differentiating KKT conditions, IntOpt \cite{hsd} proposes an interior point based approach which computes gradients by differentiating homogeneous self-dual formulation.

\textbf{Optimizing surrogate loss functions.} Elmachtoub and Grigas \cite{elmachtoub2020smart} proposed a convex surrogate loss function, namely SPO+, measuring the decision error induced by a prediction, which can handle polyhedral, convex and mixed integer programs with linear objectives. TOPNet \cite{ijcai2020-617} proposes a learned surrogate approach for exotic forms of decision loss functions, which however is hard to generalize to handle constrained programs.

Differentiating $\argmin$ is critical for gradient methods to optimize decision-focused prediction models. Many kinds of efforts, including regularization for specific problems (e.g., differentiable dynamic programming \cite{dp-differentiable}, differentiable submodular optimization \cite{diff-submodular}), reparameterization \cite{reparameterization-diff}, perturbation \cite{perturbed-opt} and direct optimization (\cite{pmlr-v48-songb16}), are spent for optimization with discrete or continuous variables and still actively investigated. 


As a comparison, our work proposes a surrogate loss function for constrained linear and quadratic programs extended with soft constraints, where the soft constraints were not considered principally in previous work. Also, unlike OptNet \cite{optnet} and CVXPY \cite{cvxpy}, our method does not need to solve KKT conditions. Instead, by adding reasonable costs for infeasibility, we convert the constrained optimization problem into an unconstrained one while keeping the same solution set, and then derive the required Jacobian analytically. To some degree, we reinvent the exact function \cite{bertsekas1997nonlinear} in prediction+optimization.


\section{Conclusion}
In this paper, we have proposed a novel surrogate objective framework for prediction+optimization problems on linear and semi-definite negative quadratic objectives with linear soft constraints. 
The framework gives the theoretical bounds on constraints' multipliers, and derives the closed-form solution as well as their gradients with respect to problem parameters required to predict by a model. 
We first convert the hard constraints into soft ones with reasonably large constant multipliers, making the problem unconstrained, and then optimize a piecewise surrogate locally. We apply and empirically validate our method on three applications against traditional two-stage methods and other end-to-end decision-focused methods. We believe our work is an important enhancement to the current prediction+optimization toolbox. There are two directions for the future work: one is to seek solutions which can deal with hard constraint parameters with negative entries, and the other is to develop the distributional prediction model rather than existing point estimation, to improve robustness of current prediction+optimization methods.

\newpage

\bibliography{neurips_2021}







\section*{Checklist}


\begin{enumerate}

\item For all authors...
\begin{enumerate}
  \item Do the main claims made in the abstract and introduction accurately reflect the paper's contributions and scope?
    \answerYes{}
  \item Did you describe the limitations of your work?
    \answerYes{}
  \item Did you discuss any potential negative societal impacts of your work?
    \answerYes{}
  \item Have you read the ethics review guidelines and ensured that your paper conforms to them?
    \answerYes{}
\end{enumerate}

\item If you are including theoretical results...
\begin{enumerate}
  \item Did you state the full set of assumptions of all theoretical results?
    \answerYes{}
	\item Did you include complete proofs of all theoretical results?
    \answerYes{}
\end{enumerate}

\item If you ran experiments...
\begin{enumerate}
  \item Did you include the code, data, and instructions needed to reproduce the main experimental results (either in the supplemental material or as a URL)?
    \answerYes{}
  \item Did you specify all the training details (e.g., data splits, hyperparameters, how they were chosen)?
    \answerYes{}
	\item Did you report error bars (e.g., with respect to the random seed after running experiments multiple times)?
    \answerYes{}
	\item Did you include the total amount of compute and the type of resources used (e.g., type of GPUs, internal cluster, or cloud provider)?
    \answerYes{}
\end{enumerate}

\item If you are using existing assets (e.g., code, data, models) or curating/releasing new assets...
\begin{enumerate}
  \item If your work uses existing assets, did you cite the creators?
    \answerYes{}
  \item Did you mention the license of the assets?
    \answerYes{}
  \item Did you include any new assets either in the supplemental material or as a URL?
    \answerYes{}
  \item Did you discuss whether and how consent was obtained from people whose data you're using/curating?
    \answerYes{}
  \item Did you discuss whether the data you are using/curating contains personally identifiable information or offensive content?
    \answerYes{}
\end{enumerate}

\item If you used crowdsourcing or conducted research with human subjects...
\begin{enumerate}
  \item Did you include the full text of instructions given to participants and screenshots, if applicable?
    \answerNA{We did not use crowdsourcing, nor did we conduct research with human subjects.}
  \item Did you describe any potential participant risks, with links to Institutional Review Board (IRB) approvals, if applicable?
    \answerNA{}
  \item Did you include the estimated hourly wage paid to participants and the total amount spent on participant compensation?
    \answerNA{}
\end{enumerate}

\end{enumerate}


\newpage
\appendix

\section{Mathematical Proofs}\label{app:math}

Throughout the following proofs, for the convenience of description, we denote two symbols:
\begin{itemize}
    \item the symbol $\angle(a,b)$, which means the angle of two vectors $a,b \in \mathbb{R}^n$;
    \item $cone(A)$, which means the conic combination of the row vectors of matrix $A$; $cone(A,B)$ stands for the conic combination of the union of the row vectors of matrix $A,B$.
\end{itemize}

Throughout this section, we will assume the same variables as in Eq.\ref{eq:prob}-\ref{eq:lagrangian} and Section~\ref{sec:methodology} by default, and further make the following assumptions or simplification for ease of description. 
\begin{itemize}
    \item Assume the main objective $g(x, \theta) = \theta^Tx$, where $||\theta||_2 \leq E$. In reality, elements of the predictive vector $\theta$ have a known range and thus $l_2$-norm of $\theta$ is bounded by a constant, namely $E \in \mathbb{R}$.
    \item We omit the soft constraints in original objective. Since all proofs in this section consider the worst direction of $\frac{\partial L}{\partial x}$ in Eq.\ref{eq:lagrangian}, for soft constraints which are linear, their derivatives are constant and can be integrated into $E$ for any bounded penalty multiplier $\alpha$.
    \item Similarly, we omit quadratic term $Q$. At first glance, it seems that $Q$ will bring unbounded derivative of decision variable $x$. However, note that we are finding the maximum value of function, and the matrix $Q$ is semi-definite positive, we have $\frac{\partial -x^TQx}{\partial x_i}<0$ for any $i$. Thus, denote the maximum point of $-x^TQx$ to be $x_0$, we can find a finite radius $r>0$, such that any point $x\in\mathcal{R}^n$ out of the circle with $x_0$ as the center and $r$ as radius is impossible to be the global optimal point; intuitively, the norm of the derivative of $Q$ becomes so large as to "draw the optimal point back" from too far, and thus we can ignore the points outside the circle. Since the circle is a bounded set, the norm of derivative $E$ is also bounded inside the circle.
    \item When we consider the effect of violating constraints, we only consider non-degraded activated cone, \ie{}, for an activated cone in a $n$-dimensional Euclidean space, there are at least $n$ non-redundant active constraints. Otherwise, as we stated in the main paper, we can project the activated cone $\mathcal{K} = \{z \in \mathbb{R}^n | z-x_0\in cone(A')\}$ onto the complementary space of the cone tip $x_0$ (\ie{} the solution space of $A'$) and proceed on a low-dimensional subspace.
    \item As for constraint itself, we assume that for any row vector of $A$ in $Ax\leq b$, and any row vector of $B$ in $Bx=c$, we have $||A_i||_2=1$ and $||B_i||_2=1$ for any row vector $A_i, B_i$, $i\in\{1,2,...,m\}$ where $m$ is the number of constraints. If not, we can first normalize the row vectors of $A$ and $B$ while adequately scaling $b, c$, then apply our proofs.
\end{itemize}
Thus it is straightforward to apply our proofs to the general form of Eq.\ref{eq:lagrangian}.

\subsection{Lemma 1}

\begin{customthml}{1}
(Bounding the utility gain) Let $R=f(x,\theta)-\max_{x_1\in\mathcal{C}}f(x_1,\theta)$ be the utility gain, then $R\leq f(x,\theta)-f(x_0, \theta)\leq \frac{E}{\cos p_0}\sum_{i=1}^{n}(A'_ix-b'_i)$, where $x$ is an infeasible point, $A'x \leq b'$ the active constraints at $x$, $p_0 = \angle({A'_i}^*, \theta^*)$ where ${A'_i}^*$ and $\theta^*$ are the optimal solution of $\max_{\theta}\min_{A'_i} cos \angle(A'_{i}, \theta)$ (i.e., the maximin angle between $\theta$ and hyperplanes of the activated cone $\mathcal{K} = \{z \in \mathbb{R}^n | z-x_0\in cone(A')\}$), and $x_0$ the projection of $x$ to the tip of cone $\mathcal{C} = \{ z \in \mathbb{R}^n | A'z \leq b'\}$. $E$ is the upper bound of $||\theta||_2$.
\end{customthml}
\begin{proof}
The first inequality is trivial as $x_0\in\mathcal{C}$ and directly followed from the definition of utility gain $R$.

For the second inequality, by the law of sines, we have $||x-x_0||_2=\frac{d_i}{\sin \tau_i}$, where $d_i$ is the distance of $x$ to the $i$-th hyper-plane in the activated cone $cone(A')$ with $A'_i$ (\ie{} the $i$-th row of $A'$) being the normal vector of the $i$-th hyper-plane in $A'x\leq b'$, and $\tau_i$ is the angle between $i$-th hyper-plane and $\theta$. Then the utility gain $R$ from violating the constraints follows: 
\begin{equation}
\begin{aligned}
R&\leq \theta^T(x-x_0)\\
 &\leq||\theta||_2||x-x_0||_2\\
 &=||\theta||_2\frac{d_i}{\text{sin}~\tau_i}\ (\text{By the law of sines; $i$ is arbitrary})\\
 &=||\theta||_2\frac{d_j}{\text{sin}~\tau_j}\ (\text{By selecting $j = \argmin_{i}{\text{sin}~\tau_i} = \argmax_{i}{cos \angle(A'_i, x-x_0)}$})\\
 &\leq E\frac{\sum_{i=1}^{n}d_i}{\text{sin}~\tau_j}\text{ (E is the upper bound of } ||\theta||_2\text{)} \\
 &= E\frac{\sum_{i=1}^{n}(A'_ix-b'_i)}{\text{sin}~\tau_j}  (\text{$d_i=\frac{|A'_ix-b'_i|}{||A'_i||_2}=A'_ix-b'_i$; note the fifth assumption.})
\end{aligned}
\end{equation}
Thus, setting all entries of penalty multiplier $\beta$ in the hard constraint conversion to $O(\frac{E}{sin~\tau_j}) = O(\frac{E}{cos \angle(A'_j, x-x_0)})$ will give us an upper bound of the minimum feasible $\beta$. 
\end{proof}
The rest of the work is to find the lower bound for $\min_{\theta}\max_i  cos \angle(A'_i, x-x_0)$ for any given activated cone $cone(A')$ and $x_0$; as we have assumed $||A'_i||=1\ \forall i\in\{1,2,...,n\}$, the objective $\min_{\theta}\max_i  cos \angle(A'_i, x-x_0)$ can be equivalently written as\footnote{With an abuse of notation, the $\theta$ in Eq.\ref{eq:minimaxprob} represents $x-x_0$, as the worst situation for deriving upper bound of $\beta$ is that $x-x_0$ has the same direction with the objective $\theta^Tx$'s derivative $\theta$.}
\begin{equation}
\label{eq:minimaxprob}
F=\min_{\theta}\max_i A_i'\theta
\end{equation}

Deriving the lower bound of $F$ is the core part of proving Theorem 3 and 6. 

\subsection{Lemma 7}
In order to prove Theorem 2, we give a crucial lemma.
\begin{lemma}
Consider a set of hyper-planes with normal vectors $A'_1,A'_2,...,A'_n\in\mathbb{R}^{n}\geq 0$, $||A'_i||_2=1$ for $i\in\{1,2,...,n\}$ which forms a cone. Let $d_i$ be the distance of a point $x$ to the hyper-plane $A'_i$, where $x$ is in the cone of $A'_i$ (\ie{} $x=k_1A'_1+...+k_nA'_n$, $k_i\geq 0 ~\forall i\in\{1,2,...,n\}$), then we have $||x||_2\leq \sum_{i=1}^{n}d_i$.
\end{lemma}
\begin{proof}
Without loss of generality, let $x=k_1A'_1+k_2A'_2+...+k_nA'_n$, and $||x||_2=1$ (we can scale $x$ if $||x||_2\neq 1$). The distance $d_i=\frac{{A'_i}^Tx}{||A'_i||_2}={A'_i}^Tx$. Therefore, we have
\begin{equation}
\begin{aligned}
    \sum_{i=1}^nd_i&=\sum_{i=1}^n{A'_i}^T(k_1A'_1+...+k_nA'_n)\\
    &=\sum_{j=1}^n k_j({A'_j}^TA'_j+\sum_{i=1,i \neq j}^n {A'_i}^T A'_{j})\\
    &\geq \sum_{j=1}^n k_j ||A'_j||_2\text{ ($A'_i\geq 0$, therefore ${A'_i}^TA'_j\geq 0$)}\\
    &= k_1+k_2+...+k_n 
\end{aligned}
\end{equation}
We next prove that $k_1+k_2+...+k_n\geq 1$. As $||x||_2=1$, we have
\begin{equation}
\begin{aligned}
||x||_2&=1\\
k_1{A'_1}^Tx+k_2{A'_2}^Tx+...+k_n{A'_n}^Tx&=1\\
\end{aligned}
\end{equation}
As ${A'_i}^Tx\in[0, 1]$, $k_1+...+k_n\geq 1$ must hold. Therefore, $\sum_{i=1}^n d_i\geq k_1+...+k_n\geq 1=||x||_2$.
\end{proof}

\subsection{Theorem 2}
\begin{customthm}{2}
 Assume the optimization objective $\theta^Tx$ with constraints $Ax\leq b$, where $A\geq 0$, and $b\geq 0$. Then, the utility gain $R$ obtained from violating $Ax-b \leq 0$ has an upper bound of $O(\sum_i\max(w_i, 0)E)$, where $w=A'x-b'$, and $A'x \leq b'$ is the active constraints.
\end{customthm}
\begin{proof}
Let $d_i=(d_{i,1}, d_{i,2},...,d_{i,n})^T$ be the vector of distances to the $i$-th hyper-plane of $A'$ with normal vector $A'_i$ from $x$. By definition of $x$ and $x_0$, the distance vector of the point to the convex hull $\mathcal{C}$ is $x-x_0$. When $A\geq 0$ and $b\geq 0$, obviously Lemma 7 holds. Consider the active constraints $w=A'x-b'\geq 0$ which is the subset of active constraints in $Ax \leq b$. The utility gain $R$ from violating the rules $w=A'x-b'$ satisfies:
\begin{equation}
\begin{aligned}
R&\leq \theta^T(x-x_0)\\
&\leq ||\theta||_2||(x-x_0)||_2\\
&\leq ||\theta||_2\sum_i||d_i||_2\text{ (Lemma 7)}\\
     &\leq E\sum_i||d_i||_2\\
     &=E\sum_i\max(w_i, 0)\text{ ($d_i=\frac{|A'_ix-b'_i|}{||A'_i||_2}=A'_ix-b'_i$; note the fifth assumption.)}
\end{aligned}
\end{equation}
\par This holds for any $x$. Therefore, we turn the hard constraint $Ax \leq b$ into a soft constraint with penalty multiplier $\beta=E$. 
\end{proof}

\subsection{Theorem 3}
\begin{customthm}{3}
Assume at least one constraint in $x\geq 0$ is active, then the utility gain $R$ by deviating from the feasible region is bounded by $O(\frac{n^{1.5}E\sum_{i}max(w_i, 0)}{\sin{p_0}})$, where $p_0=\min_{i,j}\angle(A_i, e_j)$ (\ie{} the smallest corner between axes and other constraint hyper-planes), and $w = \begin{bmatrix} A'\\ -I' \end{bmatrix} x - \begin{bmatrix} b' \\ 0\end{bmatrix}$, where $A'x \leq b'$ and $-I'x \leq 0$ are active constraints in $Ax \leq b$ and $x \geq 0$ respectively.
\end{customthm}

\begin{proof}
%
%

As Lemma 1 implies, the key point of proving this theorem is to prove that $\theta$ (the derivative of the objective), is always close to some rows of $A'$ and $-I'$. 

Consider the activated cone $\mathcal{K}=\{z \in \mathbb{R}^n | z-x_0\in cone(A', -I')\}$; all hyper-planes of $A'x = b'$ and violated entries of $x\geq 0$ must pass through $\mathcal{K}$'s tip $x_0$. Figure \ref{fig:thm3} gives such an illustrative example. For the rest of the proof, it is enough to assume that the activated cone $\mathcal{K}$ is a non-degraded cone, since otherwise, as the main paper states, we can project the activated cone onto the supplementary space of its tip. By the projection, we actually reduce the problem to the same one with lower dimensions of $x$, and can apply the following proof in the same way. Let $q$ be the number of rows in $I'$, i.e., the number of active constraints in $x \geq 0$; and $r$ be the number of rows in $A'$, i.e., the number of active constraints in $Ax \leq b$. Since $\mathcal{K}$ is non-degraded, we have $r+q \geq n$.

Without loss of generality, we transform the problem equivalently by rearranging the order of dimensions of $x$ so that $I'$ is the first $q$ rows of $I_{n \times n}$. Let's see an example after this transformation. As shown in Figure \ref{fig:thm3}, the three solid black vectors are $(-1, 0, 0)$, $(0, -1, 0)$, and $A'_1\geq 0$, where the first and second entries of $A'_1$ are strictly greater than $0$ (otherwise the cone is degraded). Similarly, for the $n$-dimensional non-degraded activated cone, we consider the maximin angle between $\theta$ and all active inequality constraints $A'$ and $-I'$ (which is equivalent to finding the minimax cosine value of angles between $\theta$ and all active inequality constraints). 
\begin{figure}
    \centering
    \includegraphics[width=6cm]{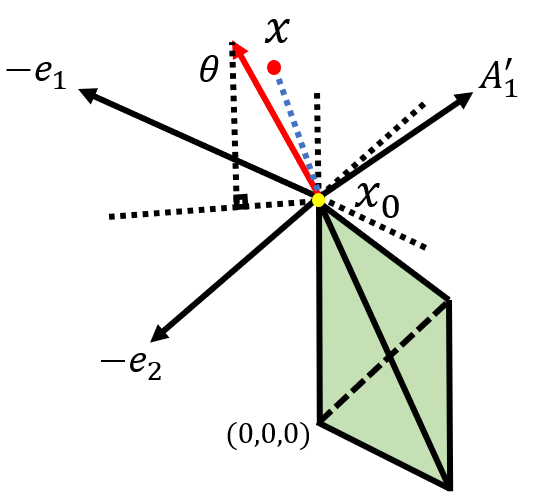}
    \caption{A $3$-dimensional illustration of the proof for Theorem 3. The green polytope is the feasible region; the black solid vectors are the related vectors (i.e. rows of $A'$) and the red line is the utility vector $\theta$. Apparently, $(0, 0, 1)^T\theta\geq 0$, otherwise the intersection point will move down to origin. The worse case appears when the red $\theta$ is on the circumcenter of the triangle formed by the three solid unit vectors.}
    \label{fig:thm3}
\end{figure}
Then, guided by Equation~\ref{eq:minimaxprob} in Lemma 1, the upper bound of distance between $\theta$ and hyper-planes with normal vector $A'_j$, is the solution of the following optimization problem w.r.t. $\theta\in\mathbb{R}^n$ and $A'_j$ (since the cone is non-degraded, we assume that $A'_j$ are linearly independent):
\begin{equation}
\label{eq:minimizeF}
\begin{aligned}
F &= \min_\theta\max_{j\in J=\{1,..,r\},k\in K=\{1,..,q\}} \{A'_j\theta, -e_k^T\theta\}, \\ s.t. ~||\theta ||_2&=1, \theta\in\mathcal{K}, r+q\geq n.
\end{aligned}
\end{equation}
where $e_k$ is the unit vector where the $k$-th entry is $1$ and others are $0$. Note that according to our presumptions, $A'_j\geq 0$, $||A'_j||_2=1$ for $j\in\{1,2,...,r\}$ have already satisfied. The constraints $\theta\in\mathcal{K}$ come from the following consideration: if with given $\theta$ we update $x$ in an iterative manner such as gradient ascending, $x$ would eventually leave current activated cone where $\theta_i<0 ~~\text{for}~~ i\in\{q+1,...,n\}$ (see Figure~\ref{fig:thm3} for an illustration). Since for any dimension index $i \in \{q+1, ..., n\}$, all entries of $A'_j$ and $-e_k$ ($j\in J,k\in K$) are non-negative, we know that $\theta_i\geq 0, ~\forall i \in \{q+1, ..., n\}$. To further relax the objective $F$ and derive the lower bound, we will assume $r+q=n$ in the rest of the paper; if $r\geq n-q$, we can simply ignore all $A'_j$ in the max operator with $j>n-q$.

Given the signs of $\theta$, we next derive a lower bound of $F$ with respect to any given $A'_j$ and set of $k$. Now we relax $F$ by setting all entries of $A'_j$ to $0$, except for the entries with index $\{1,2,...,q\}$ and $s_j$ to get $A''_j$, where $\{s_j\}(j\in \{q+1,...,n\})$ is a permutation of $\{q+1, q+2...,n\}$. We have: 
 \begin{equation}
 \label{eq:thm3}
 \begin{aligned}
     F&\geq \min_{\theta}\max_{j\in \{1,..,r\},k\in \{1,..,q\}}\{{A''}_j^T\theta, -e_k^T\theta\}\text{ (Relaxing within the $\max$ operator; note the sign of $\theta$)}\\
     &\geq \min_{\theta}\max_{j\in \{1,..,r\},k \in \{1,..,q\}}\{\sum_{i=1}^{q}A'_{j,i}\theta_i+A'_{j, s_j}\theta_{s_j}, -\theta_k\}\\
     &\geq \min_{\theta}\max_{s_j\in \{q+1, n\}, k\in \{1,..,q\}}\{\sum_{i=1}^{q}\theta_i+\alpha_0\theta_{s_j}, -\theta_k\} 
\end{aligned}
 \end{equation}
where $||\theta||_2=1$, $\alpha_0$ is the smallest non-zero entry among $A'_{j,s_j}$ (and note that $A\geq 0$ and $\theta_{s_j}\geq 0$, which means $\alpha_0\theta_{s_j}\geq 0$). The inequalities hold for the relaxation within the $\max$ operator. Note that the last line of inequality assumes $\theta_i^*\leq 0$ for $i\in\{1,2,...,q\}$ for the optimal $\theta^*$; otherwise, we can scale further, ignore the index $i$ of the optimal point $\theta^*$ where $\theta^*_i>0$ and proceed with $n-1$ dimensions for the following two facts:
\begin{enumerate}
    \item $\sum_{i=1}^q\theta_i+\alpha_0\theta_{s_j}$ term becomes smaller after ignoring such dimension;
    \item $\theta\in\mathcal{K}$, which means for any $\theta$ and any vector $y$ of the cone we have $\theta^Ty\geq 0$, and the last line still corresponds to a cone. Therefore the optimal value of last line is no less than $0$; the removal of $-\theta_k$ term given $\theta_k>0$ will not affect the optimal value.
\end{enumerate}

Then, according to the property of minimax, $\sum_{i=1}^{q}\theta_i+\alpha_0\theta_{s_j}$ should be equal to $-\theta_k$; Otherwise, if $\sum_{i=1}^{q}\theta_i+\alpha_0\theta_{s_j}$ is larger than $-\theta_k$, we can adjust the value of $\sum_{i=1}^{q}|\theta_i|$ and $\sum_{j=q+1}^{n}|\theta_j|=\sum_{j=q+1}^{n}\theta_j$ by a small amount such that the result is better and $||\theta||_2=1$ is still satisfied, which can be repeated until no optimization is possible, and vice versa. With $\sum_{i=1}^{q}|\theta_i|$ and $\sum_{j=q+1}^{n}\theta_j$ fixed, the entries of $\theta$ are equal to each other within each group by symmetry.

Therefore, the solution $\theta$ should be in the form of ($\frac{c_1}{\sqrt{n}}, \frac{c_1}{\sqrt{n}}, ..., \frac{c_2}{\sqrt{n}}, \frac{c_2}{\sqrt{n}}$) where $c_1<0, c_2>0$, and we have the following set of equations:
\begin{equation}
\label{eq:thm3_2}
\begin{aligned}
qc^2_1+(n-q)c^2_2&=n\\
qc_1+\alpha_0c_2&=-c_1
\end{aligned}
\end{equation}

With equations in Eq. \ref{eq:thm3_2}, we get $c_2=-(1+\alpha_1q)c_1/\alpha_0$, $c^2_1(q+(n-q)\frac{(1+q)^2}{\alpha^2_0})=n$. As $q\leq n$, and as $||\alpha_j||=1$, $\alpha_0$ can be further relaxed to the sine of the smallest angle $p_0$ between axes $x\geq 0$ and the other inequality constraints, and thus we have the lower bound $F\geq\frac{c_1}{\sqrt{n}}=O(\sin p/n^{1.5})$ for the optimization problem listed in Equation \ref{eq:thm3}, which is also the cosine lower bound to the nearest normal vector $A'_j$ and $-e_k$ of the activated cone. This is the denominator of the desired result; the final result follows as we apply such bound to Lemma 1. 
\end{proof}




\subsection{Corollary 4}
\begin{customthmb}{4}
For binary constraints where the entries of $A$ (before normalization) are either $0$ or $1$, the utility gain $R$ of violating $x\geq 0$ constraint is bounded by $O(n^{1.5}E\sum_{i}\max(w_i, 0))$, where $w$ is the same as Theorem 3.
\end{customthmb}

\begin{proof}
The proof of Corollary 4 is almost the same with Theorem 3, except that if the constraint $A$ is cardinal, then each entry of $A$ is either $0$ or $1$. Thus, 
as all the non-zero entry are the same, we shall replace the third line in Equation \ref{eq:thm3} with
\begin{equation}
\min_{\theta}\max_j\{\frac{\sum_{i=1}^{q}\theta_i}{\sqrt{h_j}}+\frac{\theta_{s_j}}{\sqrt{h_j}},-\theta_k\}
\end{equation}
where $h_j$ is the number of non-zero entry of $A'_j$, each entry being $\frac{1}{\sqrt{h_j}}$ as normalized $||A'_j||_2=1$.
By substituting $h_j$ with $n$\footnote{Note in Theorem 3 we have already mentioned the non-positivity of $\theta_i$ where $i\in\{1,2,...,q\}$, thus bigger $h_j$ brings smaller objective.}, this optimization problem can be further relaxed to 
\begin{equation}
\min_{\theta}\max_j\{\frac{\sum_{i=1}^{q}\theta_i}{\sqrt{n}}+\frac{\theta_{s_j}}{\sqrt{n}},-\theta_k\}    
\end{equation}
and we can change Equation \ref{eq:thm3_2} to
\begin{equation}
\begin{aligned}
qc^2_1+(n-q)c^2_2&=n\\
qc_1+c_2&=-\sqrt{n}c_1
\end{aligned}
\end{equation}
Therefore, similar to the proof of Theorem 3, we get $\frac{c_1}{\sqrt{n}}=\sqrt{\frac{1}{q+(n-q)(\sqrt{n}+q)^2}}$, which leads to the bound $O(n^{1.5}E\sum_i\max(w_i, 0))$, removing the $\sin p$ in the denominator. 
\end{proof}

\subsection{Theorem 5}
\begin{customthm}{5}
If there is only one equality constraint $B^T x=c$ (e.g. $\sum_{i}x_i=1$) and special inequality constraints $x\geq 0$, $Ix \leq b$, 
then the utility gain $R$ from violating constraints is bounded by $O(\frac{n^{1.5}E\sum_{i}\max(w_i, 0)}{\sin{p}})$, where $p$ is the same with theorem 3, $w$ is the union of active $B^Tx-c$ and $-x$.\footnote{See Theorem 3 for the meaning of union.}
\end{customthm}

\begin{proof}
We first prove the situation where the inequality constraint is only $x\geq 0$. If we only have one equality constraint, then it can be seen as two separate inequality constraints, which are $B^Tx\leq c$ and $-B^Tx\leq -c$; for any non-degraded activated cone, as $B^Tx<c$ and $B^Tx>c$ cannot hold simutaneously, there is at most one normal vector of constraint in the activated cone. 

If $B^Tx\leq c$ is active (\ie{} this constraint is violated, now we have $B^Tx>c$), then with $B\geq 0$, the case is exactly the same with Theorem 3. Otherwise, if $-B^Tx\leq -c$ (\ie{} $B^Tx\geq c$) is active (\ie{} $B^Tx\geq c$ is violated, now $B^Tx<c$), then the normal vector of the current active constraint  $-B^Tx\leq -c$ is $-B\leq 0$. Note that all other normal vectors of active constraints are $-e_k\leq 0$ (which is the same situation as that in Equation \ref{eq:minimizeF} in the proof of Theorem 3, except that $A'_j$ is substituted with $-B^T\leq 0$); this indicates that the condition of Lemma 7 is satisfied under such scenario. Therefore, we can apply the proof of Theorem 2 in this case and get a better bound than that in the first scenario. 


Then, the full theorem is proved as follows: consider any activated cone $\mathcal{K}$. For any $i\in\{1,2,...,n\}$, $x_i\leq b_i$ and $x_i\geq 0$ cannot be active simultaneously. If the former is activated, we replace $x_i$ with $b_i-x_i$; if the latter (or neither) is active, we remain $x_i$ as normal. Then, for this activated cone, the scenario is exactly the same with the situation where the inequality constraints are only $x\geq 0$. Similarly, if the equality constraint is cardinal, with the same proof of Corollary 4, we can remove the $\sin p$ in the denominator.

\end{proof}

\subsection{Theorem 6}
\begin{customthm}{6}
Given constraints $Ax\leq b$, $x\geq 0$, and $Bx=c$, where $A, B, b, c\geq 0$, the utility gain $R$ obtained from violating constraints is bounded by $O(\sqrt{n}\lambda_{max}\sum_{i}\max(w_i, 0))$, where $\lambda_{max}$ is the upper bound for eigenvalues of $P^TP$ ($\mathcal{P}: x\rightarrow Px$ is an orthogonal transformation for an $n$-sized subset of normalized row vectors in $A, B$ and $-I$), and $w$ is the union of all active constraints from $Ax\leq b$, $x\geq 0$, $Bx\leq c$ and $-Bx\leq -c$. 
\end{customthm}

\begin{proof}
For any non-degraded activated cone $\mathcal{K}$, we have the following optimization problem (it is worth noting that reformulating our original problem to this optimization one is inspired by Equation~\ref{eq:minimaxprob} in Lemma 1):
\begin{equation}
\begin{aligned}
F=\min_\theta\max_{j\in J,k\in K}\{\max_{i\in D,n_i\in\mathbb{R}^{n}}\ & \{n_i^T\theta\}, A'_j\theta, -e_k^T\theta\} \\
s.t. ||\theta ||_2=1, \theta\in\mathcal{K}\\
d+r+q\geq n, ||n_i||_2=1 &~ \forall i \in D=\{1,2,...,d\},<B_1, ..., B_d>=<n_1,...,n_d>
\end{aligned}
\end{equation}where $A'_j$ is the normal vector of the $j$-th inequality constraints in $Ax\leq b$, $n_i$ is the $i$-th normal vector for the subspace of $Bx=c$ (and thus represents the same subspace that is represented by the rows of $B$). According to the assumption, we scale $Ax\leq b$ and $Bx=c$ so that $||n_i||=||A'_j||=1$ for any $i,j$. Note that different from the optimization problem stated in Theorem 3, in this problem $q$ can be $0$, which means that there can be no entry of $x\geq 0$ active. To get the lower bound of $F$, we may relax the function by keeping the normal vectors $\{n_i\}$ in their original direction $B$ to 
$F\geq G=\min_\theta\max_{i\in D,j\in J,k\in K}\{B'_i\theta, A'_j\theta, -e_k^T\theta\}$
where the $i$-th row of $B$ satisfies $B'_i=B_i$ or $B'_i=-B_i$ (the sign of $B'_i$ is decided by the direction of $\theta$), and thus $B'_i\geq 0$ or $B'_i\leq 0$ (\ie{} any pair of elements in $B'_i$ will not have opposite signs) for any $i$. Moreover, $||B'_i||_2=1$. We denote the positive $B'_i$ as $B^p_{i_1}$, $i_1\in D_1$ and the negative $B'_i$ as $B^n_{i_2}$, $i_2\in D_2$. Then similar to Theorem 3,
\begin{equation}
G=\min_\theta\max_{i_1\in D_1, i_2\in D_2, j\in J, k\in K}\{B^n_{i_1}\theta, B^p_{i_2}\theta, A'_j\theta, -e_k^T\theta\}
\end{equation}

where $D_1\bigcap D_2=\emptyset$, $D_1\bigcup D_2=D$. we aggregate $A'$, $B^p$ and $-e_k$, $B^n$ into $\alpha$, and get
\begin{equation}
G=\min_\theta\max_{j\in S}\{{\alpha}_j^T\theta\}
\end{equation}
where $|S|\geq n$, $\alpha_j\geq 0$ or $\leq 0$ for any $j\in S$. 
Therefore, we can further relax $G$ by selecting $n$ linearly independent vectors that are the closest to the minimum product\footnote{The existence of such set of vectors comes from the non-degradation of the activated cone.} and ignore the others for the $\max$ operator, and for the rest of the proof we may assume that $|S|=n$.

We will next consider a linear transformation $\mathcal{P}:x\rightarrow Px$ from $\mathbb{R}^n$ to $\mathbb{R}^n$ that transforms $\{\alpha_j\}$ into an orthogonal normal basis. Then we have
\begin{equation}
\begin{aligned}
    \lambda_{max}G&=\min_{\theta}\max_j\lambda_{max}{\alpha}_{j}^{T}\theta\\
                  &\geq \min_{\theta}\max_i{\alpha}_jP^TP\theta\\
                  &=\frac{1}{\sqrt{n}}\sum_{k=1}^n{\alpha}_jP^TP{\alpha}_k\\
                  &=\frac{1}{\sqrt{n}}{\alpha}_jP^TP{\alpha}_j\\
                  &=\frac{1}{\sqrt{n}}
\end{aligned}
\end{equation}
where $\lambda_{max}$ is the maximum eigenvalue of $P^TP$. The row vectors of $A''$ are $\{\alpha'_i\}$. Therefore, our problem $G$ has a lower bound of $\frac{1}{\sqrt{n}\lambda_{max}}$ where $\lambda_{max}$ is the maximum eigenvalue for $P^TP$; $A''$ consists of $n$ of the vectors in $A'$ and has the largest maximum eigenvalue for $P^TP$. 
\end{proof}

Though we do not derive the bound of $\lambda_{max}$ with respect to the number of dimension $n$ and the angle $p_0$ between hyper-planes in the activation cone, we empirically evaluate the behavior of $\lambda_{max}$ with respect to $n$ on randomly generated data. We first generate a normal vector $n\in R^n$, with each entry generated independently at random from the distribution $D$; then we generate $n$ vectors $\{\alpha_1, ..., \alpha_n\}$ in $R^n$ with their entries either all not smaller than $0$ or all not greater than $0$; the orthant is chosen with probability $0.5$. Each entry is independently generated from $D$ and shifted by a constant to enforce the signs. We ensure that $\forall i$, $\alpha_i^Tn\geq 0$ by discarding the vectors that do not satisfy such constraint. $D$ can be uniform distribution $U(0, 1)$, the distribution for absolute value of normal distribution $|N(0, 1)|$, or beta distribution $B(2, 2)$. The $\alpha_i$ is then normalized, and we record the largest eigenvalue of $R^TR$, where $A=\begin{bmatrix}\alpha_1^T\\\alpha_2^T\\...\\\alpha_n^T\end{bmatrix}$, $A=RQ$ is the RQ decomposition of $A$. We repeat $1000$ times for each $n$ and record the mean and maximum eigenvalue for $R^TR$. Below is the result of our evaluation; it shows that the maximum eigenvalue $\lambda_{max}$ is approximately $O(n^2)$. See Figure \ref{fig:eval_uniform},\ref{fig:eval_gaussian}, and \ref{fig:eval_beta} for illustration.
\begin{figure}
    \centering
    \includegraphics[width=0.6\linewidth]{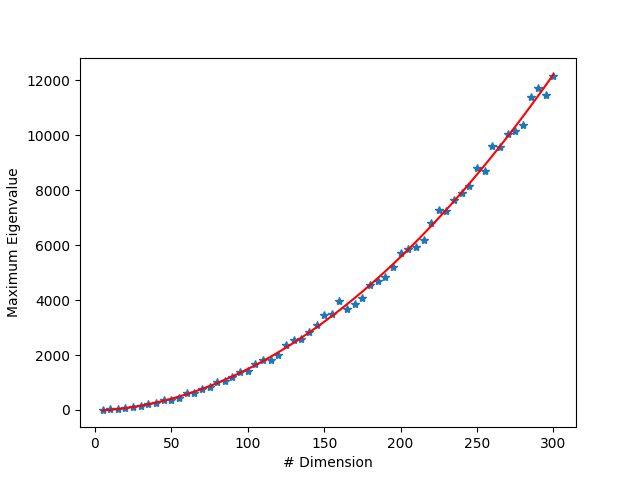}
    \caption{Empirical estimation of the max maximum eigenvalue over $1000$ trials with $D$ being uniform. The $x$ axis is $n$, and the $y$ axis is $\lambda_{max}$. The $2$-degree polynomial fitting curve is $0.1282x^2+2.392x-32.89$.}
    \label{fig:eval_uniform}
\end{figure}
\begin{figure}
    \centering
    \includegraphics[width=0.6\linewidth]{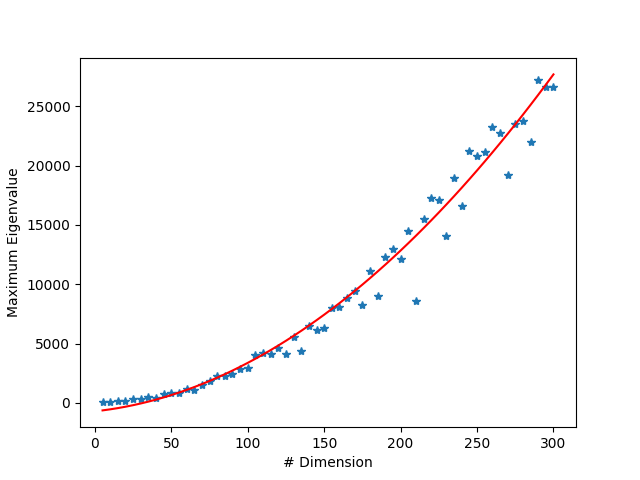}
    \caption{Empirical estimation of the max maximum eigenvalue over $1000$ trials with $D$ being Gaussian. The $x$ axis is $n$, and the $y$ axis is $\lambda_{max}$. The $2$-degree polynomial fitting curve is $0.2693x^2+13.95x-716.8$.}
    \label{fig:eval_gaussian}
\end{figure}
\begin{figure}
    \centering
    \includegraphics[width=0.6\linewidth]{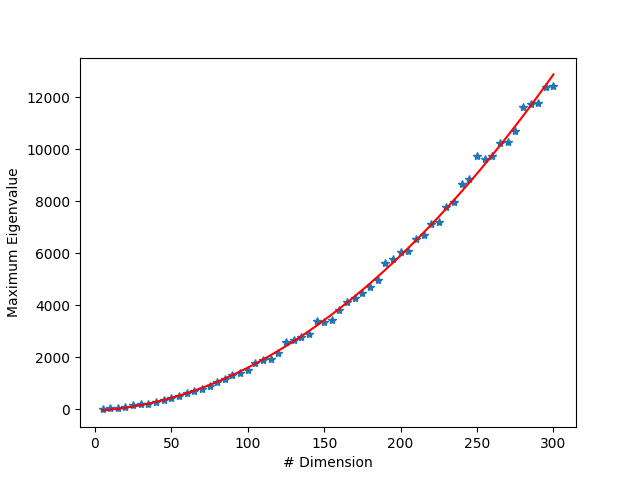}
    \caption{The max maximum eigenvalue over $1000$ trials with $D$ being $Beta(2,2)$. The $x$ axis is $n$, and the $y$ axis is $\lambda_{max}$. The $2$-degree polynomial fitting curve is $0.1325x^2+3.403x-69.95$.}
    \label{fig:eval_beta}
\end{figure}
\section{Choices of Surrogate $\max$ Functions}

The final goal of surrogate $max$ function is to relax the objective with the term $\alpha^T\max(z=Cx-d, 0)$, making it differentiable over $\mathbb{R}^n$. Such objective is a piecewise function with respect to $z$ with two segments: one is constant $0$ with derivative $0$, the other is linear with a constant derivative.

At first glance, sigmoidal surrogates are seemingly the most straightforward candidate for modeling the derivative of such a piecewise function. For a sigmoidal approximation of soft constraints, $S(z)$ should satisfy the following four conditions, among which the first two are compulsory, and the third and fourth can be slightly altered (e.g. by setting $\epsilon_2=0$ and remove the fourth condition).

\begin{enumerate}
\item When $z\rightarrow\infty$, $S'(z)\rightarrow M^-=(1+\epsilon_1)^-$. $\epsilon_1>0$ should be a small amount, and it serves as a perturbation since otherwise $g'(z)$ would be always greater than $0$. However, $\epsilon_1$ should not be too small to let the optimal point be too far away from the original constraint (for the gradient will reach $\alpha$ too late).
\item $S$ must be differentiable, and must have a closed-form inverse function.
\item When $z\rightarrow-\infty$, $S'(z)\rightarrow\epsilon_2^+<0$, where $\epsilon_2$ is constant and very close to $0$. The $S'(z)$ should cause as small impact as possible when $z<0$ (when there is no waste), and meanwhile it should (very slightly) encourage $z$ to grow to the limit instead of discouraging.
\item When $z=0$, $S'(z)=0$, which means the penalty function $S$ has no influence on the clipping border. This condition is important for linear programming; it can be removed for quadratic programming.
\end{enumerate}

To satisfy the conditions above, we need a closed-form sigmoidal function with closed-form inverse function to be $S'(z)$, the \textit{derivative} of the surrogate. Such function can be Sigmoid function ($S'(z)=\frac{1}{1+e^{-z}}$), Tanh/arctan function, or fractional function ($S'(z)=0.5(\frac{z}{\sqrt{z^2+1}}+1)$, or equivalently $S(z)=0.5\sqrt{z^2+1}+0.5z$).
With the form of $S(z)$ confirmed, our function should look like\footnote{To satisfy the requirements above, the function needs scaling and translation; we omit them for the simplicity of formulae by simply set $\alpha=1$, $\epsilon_1=\epsilon_2=0$.}
$L(x)=g(x,\theta)-\alpha^TS(z)$, with the derivative $\frac{\partial L}{\partial x}=\frac{\partial g}{\partial x}-\frac{\partial z}{\partial x}S'(z)\alpha$.

Take $g(x)=\theta^Tx$, $S(z)=0.5\sqrt{1+z^2}+0.5z$ as an example. The optimal point satisfies
\begin{equation}
\theta= C^T\text{diag}(\frac{0.5z}{(1+z^2)^{0.5}}+0.5)\alpha\\
\end{equation}
\par Let $y_i=(\frac{0.5z}{(1+z_i^2)^{0.5}}+0.5)\alpha_i$, $w_i=(\frac{0.5x}{(1+x_i^2)^{0.5}}+0.5)$, then $\theta=C^Ty$. $S(z)$ can be substituted by any function with sigmoidal function (e.g. $arctan, tanh$, and $sigmoid$) as its derivative. However, $C$ is impossible to be invertible, for $C$ includes an identity matrix representing $x\geq 0$ and must have strictly more rows than columns if $C$ has any other constraint. Therefore, we cannot solve $z$ by first solving $y$ exactly in a linear system; we need to solve the equation of fractional function, or estimate $y$ with pseudo-inverse. Hence, the closed-form solution with sigmoidal surrogate is at least non-attractive.
 \begin{figure}
      \centering
      \includegraphics[scale=0.5]{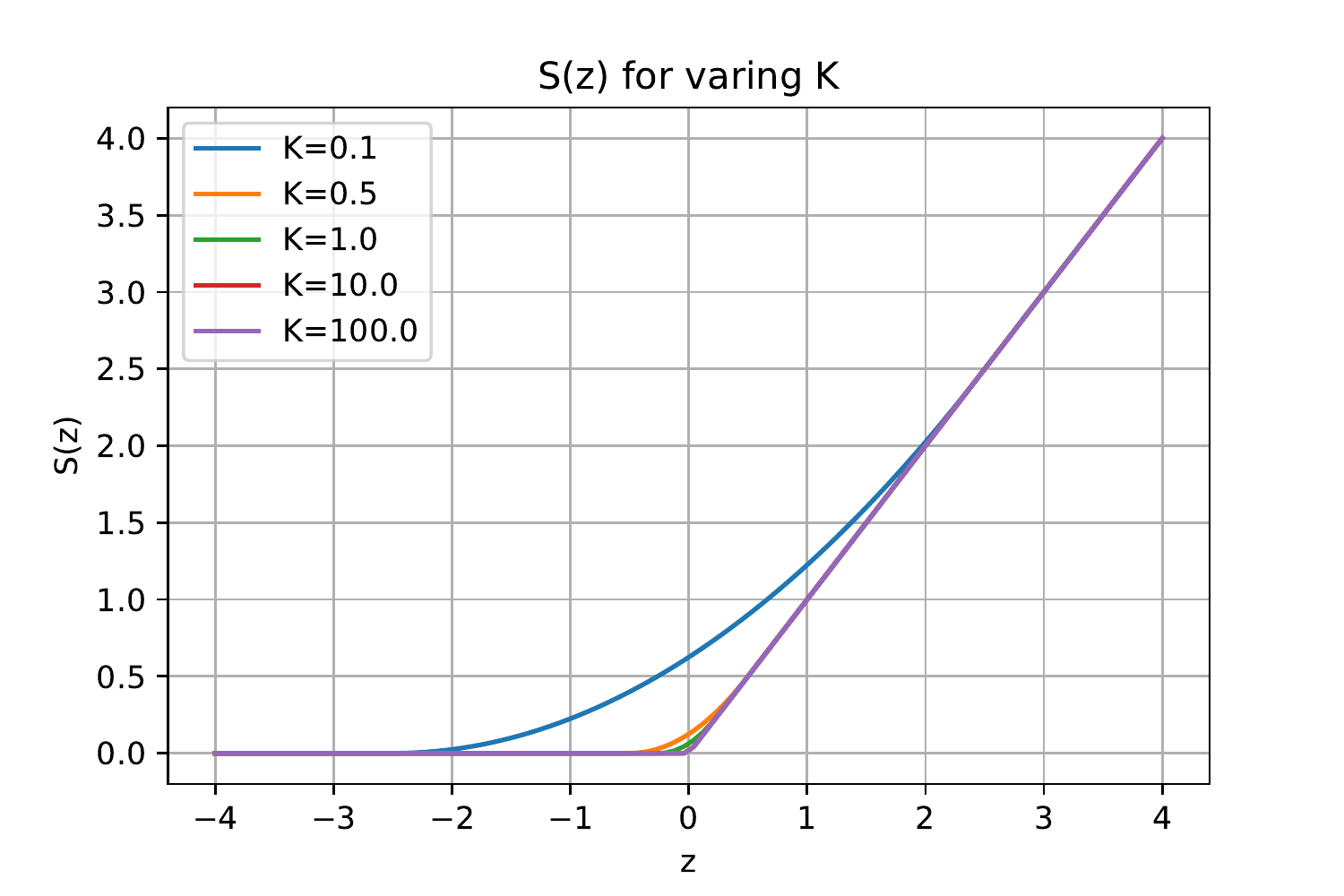}
      \caption{S(z), the proposed surrogate function of $max(z, 0)$}
      \label{fig:f_sz}
 \end{figure}

To derive a closed-form solution $x$ for the equation $0=\frac{\partial L}{\partial x}=\frac{\partial g}{\partial x}-\frac{\partial z}{\partial x}S'(z)\alpha$, $S'(z)$ should be in a simple form (e.g. polynomial form with degrees lower or equal to $4$ with respect to $x$). Unfortunately, no basic function except the sigmoid function can satisfy the conditions listed above; the degree limit makes higher-order Taylor expansion infeasible. Thus, our only choice left is to find a differentiable piecewise function with linear parts on both sides; we shall first solve the optimal point of the surrogate numerically, and extend the segment where the optimal point locates to $R^n$; the gradient $\frac{\partial x}{\partial \theta}$ is correct due to the uniqueness of optimality (which requires a convex surrogate).

Therefore, we choose the function $S(z)$ to be 
\begin{equation}
\label{eq:sz}
S(z)=\begin{cases}
          0 \quad &\text{if} \, z<-\frac{1}{4K} \\
          K(z+\frac{1}{4K})^2 \quad &\text{if} \, -\frac{1}{4K}\leq z\leq\frac{1}{4K} \\
          z \quad &\text{if} \, z\geq\frac{1}{4K} \\
     \end{cases}
\end{equation}
as illustrated in Figure \ref{fig:f_sz}, where $K$ is a constant. The surrogate is chosen for its simplicity. Not only does such choice satisfy our conditions with minimum matrix computation, but moreover, $S'(z)$ is linear, which means the equation can be solved in a linear system. The limitation of such method is that it introduces a hyper-parameter $K$ to tune: when $K$ is large, if the predicted optimal point is far from ground-truth, the gradient $\frac{\partial f(x, \theta_{\text{real}})}{\partial x}$ will be too large, making the training process instable. On the other hand, if $K$ is small, the optimal point may deviate from the feasible region too far, as the hard constraints are not fully enforced within the region $z\in[-1/4K,4K]$. Fortunately, our empirical evaluation shows that the effect of $K$ is smooth, and can be tuned by a grid search.

\section{Learning and Inference Algorithm}
\label{app:algorithm}

In general, the learning and inference procedures based on our method have the same prediction+optimization workflow as the computational graph we gave in Fig.\ref{fig:computational_graph}. In this section, we propose the detailed algorithm with emphasis on the specific steps in our method.

Alg.~\ref{alg:surrogate-learning} is the learning algorithm of stochastic gradient descent in mini-batch. For each sample $(\xi_i, \theta_i)$, it predicts the context variables as $\hat{\theta_i}$ (Statement 3) and then solve the program $f$ to get the optimal decision variable as $\hat{x_i^*}$ (Statement 4). Note that here $f$ is solved by solvers such as Gurobi and CPLEX that are capable to get the deterministic optimal solutions. With $\hat{x_i^*}$, the piece of $S(z)$ is determined, and thus the surrogate objective $\bar{f}$. From Statement 5 to 7, we decide the form of surrogate $\bar{f}$ for $\hat{x_i^*}$ by first calculating $\hat{z_i^*}$, then use $\hat{z_i^*}$ to decide $M_i$ and $U_i$ in the surrogate. $[\cdot]$ is an element-wise indicator function. Specially, if $C$ or $d$ is the predicted parameter, $\hat{z_i^*}$ should be calculated by the predicted value of $C$ or $d$; moreover, in statement 8 we should calculate $\frac{\partial r}{\partial x}$ on the estimated segment, but with ground-truth value of $C$ or $d$. In Statement 8, we compute the gradient by back-propagation. Specifically, we directly compute $\frac{\partial x}{\partial \theta}$ and $\frac{\partial r}{\partial x}$ by their analytical forms, while $\frac{\partial \theta}{\partial \psi}$ by auto-grad mechanism of end-to-end learning software such as PyTorch \cite{pytorch}, which we use throughout our experiments. Finally, in Statement 9, we update the parameters of $\psi$ with accumulated gradients in this mini-batch.

\begin{algorithm}[h]
\caption{SGD Learning with the surrogate $\bar{f}$}
\label{alg:surrogate-learning}
\SetAlgoVlined
\SetKwInOut{Input}{Input}
\SetKwInOut{Output}{Output}
\SetKw{KwBy}{by}
\Input{dataset $D = \{(\xi_i, \theta_i)\}_{i=1}^N$}
\Input{optimization settings $\{A, b, C, d, \alpha \}$}
\Input{derived hyper parameters $\{K\}$}
\Input{learning rate $a$ and batch size $s$}
\Output{the predictive model $\Phi_{\psi}$ with parameters $\psi$}
\Begin{
\sf{
  \nl {Sample mini-batch $D_k \sim D$}; \\
  \nl \ForEach{$(\xi_i, \theta_i) \in D_k$} {
	\tcp{ Estimate $\theta_i$ with the model $\Phi_{\psi}$.}
	\nl {$\hat{\theta_i} \gets \Phi_{\psi}(\xi_i)$};\\
	\tcp{ Solve the original $f$ with $\hat \theta_i$.}
	\nl {$\hat{x_i^*} \gets {\textrm{arg}}\max_{x}{f(x, \hat{\theta_i})}$};\\	
	\tcp{decide surrogate $\bar{f}$ for $\hat{x_i^*}$ by calculating $M_i$ and $U_i$}
	\nl {$\hat{z_i^*} \gets C\hat{x_i^*}-d$};\\
	\nl {$M_i\gets diag([-1/4K\leq \hat{z_i^*}\leq 1/4K])$};\\
	\nl {$U_i\gets diag([\hat{z_i^*}\geq 1/4K])$};\\
	\tcp{ Compute $\frac{\partial \hat{r_i^*}}{\partial \psi}$, where $\bar{r_i^*}=\bar{f}(\hat{x_i^*}, \theta_i)$}
	\nl {$\frac{\partial \bar{r_i^*}}{\partial \psi} \gets \frac{\partial \hat{\theta_i}}{\partial \psi} \times \frac{\partial \hat{x_i^*}}{\partial \hat{\theta_i}} \times \frac{\partial \bar{r_i^*}}{\partial \hat{x_i^*}} \mid_{\theta_i, ~\hat{\theta_i},~\hat{x_i^*}}$}
  }
  \tcp{ Update $\psi$ with ascent gradients}
  \nl {$\psi \gets \psi + a \times \frac{1}{s} \times \sum\limits_{i=1}^s \frac{\partial \bar{r_i^*}}{\partial \psi}$}
}
}
\end{algorithm}

Alg.~\ref{alg:surrogate-inference} is the program with context variable prediction for inference; it is the same with the traditional predict-then-optimize paradigm. Note that the surrogate is no longer needed in this phase.

\begin{algorithm}[ht]
\caption{optimization with predictive variables}
\label{alg:surrogate-inference}
\SetAlgoVlined
\SetKwInOut{Input}{Input}
\SetKwInOut{Output}{Output}
\Input{dataset $D = \{\xi_i\}_{i=1}^N$}
\Input{optimization settings $\{A, b, C, d, \alpha \}$}
\Input{derived hyper parameters $\{K\}$}
\Input{learned predictive model $\Phi_{\psi}$}
\Output{Solutions \{$(\hat{x_i^*}, \hat{r_i^*})$\}}
\Begin{
\sf{
  \nl \ForEach{$\xi_i \in D$} {
	\tcp{ Estimate $\theta_i$ with the model $\Phi_{\psi}$.}
	\nl {$\hat{\theta_i} \gets \Phi_{\psi}(\xi_i)$};\\
	\tcp{ Solve $f$ with the estimated $\hat \theta_i$.}
	\nl {$\hat{x_i^*} \gets {\textrm{arg}}\max_{x}{f(x, \hat{\theta_i})}$};\\
	\tcp{ Evaluate $f$ with the real $\theta_i$.}
	\nl {$\hat{r}_i^* \gets f(\hat{x_i^*}, \theta_i)$};
  }
}
}
\end{algorithm}

\section{Derivation of Gradients}
In this section, we will show the derivation process of the gradient with respect to the three problems stated in the main paper. 
\subsection{Linear Programming with Soft Constraints}\label{sub:app1}
The problem formulation is:
\begin{equation}\max_x \theta^Tx-\alpha^T\max(Cx-d, 0)\text{, s.t. } Ax\leq b\end{equation}
where $\alpha\in\mathbb{R}^n\geq 0, A\in\mathbb{R}^{m_1\times n}\geq 0, b\in\mathbb{R}^{m_1}\geq 0$, and $\theta\in\mathbb{R}^n$ is to be predicted. In this formulation and the respective experiment (synthetic linear programming), we assume that there is no equality constraint.

Let $C'=\begin{bmatrix}C\\A\\-I\end{bmatrix},d'=\begin{bmatrix}d\\b\\0\end{bmatrix},\gamma=\begin{bmatrix}\alpha\\O(\frac{n^{1.5}E}{\sin p})\\O(\frac{n^{1.5}E}{\sin{p}})\end{bmatrix}$, where $E$ is the upper bound of $||\theta||_2$, $p$ is the minimum angle between hyper-planes of $Ax\leq b$ and the axes. Then, we can write the surrogate as
\begin{align}
\begin{split}
\gamma^TS(z)&=\gamma^T(0.5M(z+\frac{1}{4K})^2+U z)\\
\end{split}
\end{align}
Calculating the derivative of $\theta^Tx-\gamma^TS(z)$ at the optimal point, we get
\begin{align}
\begin{split}
\theta&=C'^T(M(\text{diag}(z)+\frac{1}{4K}I)+U)\gamma 
         =C'^TM \text{diag}(\gamma)(C'x-d')+C'^T(\frac{1}{4K}M+U)\gamma
\end{split}
\end{align}
This equation gives us the analytical solution of $x$: 
\begin{equation}
x=(C'^TM \text{diag}(\gamma)C')^{-1}(\theta+C'^TM \text{diag}(\gamma)d'-C'^T(\frac{1}{4K}M+U)\gamma)   
\end{equation}
Based on such derivative solution of $x$, we differentiate with respect to $\theta$ on both sides: 
\begin{align}
\begin{split}
\frac{\partial x}{\partial \theta}&=(C'^TM \text{diag}(\gamma)C')^{-1}
\end{split}
\end{align}
On the other hand, given real parameter $\theta_{real}$, we calculate the derivative of $f(x, \theta_{\text{real}})$:
\begin{align}
\begin{split}
\frac{\partial f(x, \theta_{\text{real}})}{\partial x}&=\theta_{\text{real}}-C'^TM\text{diag}(\gamma)(C'x-d')-C'^T(\frac{1}{4K}M+U)\gamma
\end{split}
\end{align}

\subsection{Portfolio Selection with Soft Constraints}\label{sub:app2}

We consider minimum variance portfolio (\cite{portfolio}) which maximizes the return while minimizes risks of variance. The problem formulation is:
\begin{align}
\begin{split}
&\max_x \theta^Tx-x^TQx-\alpha^T\max(Cx-d, 0)\\
\text{s.t.} & ~ x^T\mathbf{1} = 1, ~ x \geq 0, ~ Q \geq 0, ~\alpha\geq 0\\ 
\end{split}
\end{align}
where $x\in\mathbb{R}^{n}$ is the decision variable vector -- equity weights,  $\theta\in\mathbb{R}^{n}$ is the equity returns, the semi-definite positive $Q\in\mathbb{R}^{n\times n}$ is the covariance matrix of returns $\theta\in\mathbb{R}^{n}$. Equivalently, we rewrite the constraints of $x$ to fit our surrogate framework, as $Ax \leq b$ where $A=\begin{bmatrix}-\mathbf{1}_{1\times n}\\\mathbf{1}_{1\times n}\\-I\end{bmatrix},b=\begin{bmatrix}-1\\1\\0\end{bmatrix}$. Let $C'=\begin{bmatrix}C\\A\end{bmatrix},d'=\begin{bmatrix}d\\b\end{bmatrix},\gamma=\begin{bmatrix}\alpha\\O(n^{1.5}E)\textbf{1}\end{bmatrix}$, and we have surrogate $\gamma^TS(z)=\gamma^T(0.5M(Z+\frac{1}{4K})^2+Uz)$. Then, we may derive the optimal solution $x$ and the gradient of $f(x,\theta_{\text{real}},Q_{\text{real}})$ with real data with respect to $x$ as 
\begin{equation}
\begin{aligned}
x&=(2Q+C'^TM \text{diag}(\gamma)C')^{-1}(\theta+C'^TM \text{diag}(\gamma)d'-C'^T(U+\frac{M}{4K})\gamma)\\
\frac{\partial f(x,\theta_{\text{real}},Q_{\text{real}})}{\partial x}&=\theta_{\text{real}}-2Q_{\text{real}}x-C'^TM \text{diag}(\gamma)(C'x-d')-C'^T(\frac{1}{4K}M+U)\gamma
\end{aligned}
\end{equation}
Differentiating on both sides of the analytical solution of $x$, we get:
\begin{equation}
\begin{aligned}
\frac{\partial x}{\partial\theta}&=(2Q+C'^TM \text{diag}(\gamma) C')^{-1}\\
\frac{\partial x}{\partial Q}&=(\theta+C'^TM \text{diag}(\gamma)d'-C'^T(U+\frac{M}{4K})\gamma)\frac{\partial (2Q+C'^TM \text{diag}(\gamma)C')^{-1}}{\partial Q}
\end{aligned}
\end{equation}
To derive a simplified norm of $\frac{\partial x}{\partial Q}$, let $R=(2Q+C'^TM \text{diag}(\gamma)C')^{-1}$, $S=C'^TM \text{diag}(\gamma)C'$, $\beta=\theta+C'^TM \text{diag}(\gamma)d'-C'^T(U+\frac{M}{4K})\gamma$. With such notations, we can simplify the previous results to
\begin{equation}
x=(2Q+S)^{-1}\beta=R\beta,\ 
\frac{\partial x}{\partial \theta}= R^T=R,\ 
\frac{\partial f(x,\theta_{\text{real}},Q_{\text{real}})}{\partial x}=\beta_{\text{real}}-R_{\text{real}}^{-1}x
\end{equation}
and the derivative $\frac{\partial x_i}{\partial Q_{j,k}}$ and $\frac{\partial f(x,\theta_{\text{real}},Q_{\text{real}})}{\partial Q_{j,k}}$ can be derived as follows:
\begin{equation}
\begin{aligned}
\frac{\partial x_i}{\partial Q_{j,k}}&=\sum_{x}\sum_{y}\frac{\partial R_{x,y}}{\partial Q_{j,k}}\frac{\partial x_i}{\partial R_{x,y}}\\
&=\sum_{y}\frac{\partial R_{i,y}}{\partial Q_{j,k}}\beta_{y}\\
&=\sum_{y}\beta_{y}\sum_{p}\sum_{q}\frac{\partial (2Q+S)_{p,q}}{Q_{j,k}}\frac{\partial R_{i,y}}{\partial (2Q+S)_{p,q}}\\
&=\sum_{y}2\beta_{y}\frac{\partial R_{i,y}}{\partial (2Q+S)_{j,k}}\\
&=-\sum_{y}2\beta_{y}R_{i,j}R_{k,y}
\end{aligned}
\end{equation}
\begin{equation}
\begin{aligned}
\frac{\partial f(x,\theta_{\text{real}},Q_{\text{real}})}{\partial Q_{j,k}}&=-\sum_{i}(\sum_{y}2\beta_yR_{i,j}R_{k,y})(\beta_{\text{real}, i}-(R_{\text{real}}^{-1}x)_i)\\
&=2\sum_{i}((R_{\text{real}}^{-1}x)_i-\beta_{\text{real},i})R_{i,j}\sum_{y}\beta_{y}R_{k,y}
\end{aligned}
\end{equation}
Let $p_j=\sum_{i}{R^T_{j,i}((R^{-1}_{\text{real}}x)_{i}-\beta_{\text{real}, i})}, t_k=\sum_{y}\beta_yR_{k,y}$ ($p=R^{T}(R_{\text{real}}^{-1}x-\beta_{\text{real}})$, $t=x$), and finally
\begin{equation}
\frac{\partial f(x,\theta_{\text{real}},Q_{\text{real}})}{\partial Q_{j,k}}=2p_jx_k,\ 
\frac{\partial f(x,\theta_{\text{real}},Q_{\text{real}})}{\partial Q}=2px^T
\end{equation}

\subsection{Resource Provisioning}

The problem formulation of resource provisioning is 
\begin{equation}
    \min_{x} \alpha_1^T\max (Cx-d, 0)+\alpha_2^T\max (d-Cx, 0),\ s.t. ~ x^T\mathbf{1}=1, x\geq 0
\end{equation}

Let $C'\begin{bmatrix}C\\-C\\-I\\\mathbf{1_{1\times n}}\\\mathbf{-1_{1\times n}}\end{bmatrix}$, $d'=\begin{bmatrix}d\\-d\\0\\1\\-1\end{bmatrix}$, $\gamma=\begin{bmatrix}\alpha_1\\\alpha_2\\O(n^{1.5}E)\mathbf{1}\\O(n^{1.5}E)\\O(n^{1.5}E)\end{bmatrix}$, $P=(C'^TM \text{diag}(\gamma)C')^{-1}$, $ \beta=(C'^TM \text{diag}(\gamma)d'-C'^T(\frac{M}{4K}+U)\gamma)$, $\eta=M \text{diag}(\gamma)d'-(\frac{M}{4K}+U)\gamma$. Then for the derivative $ \frac{\partial f(x,C'_{\text{real}})}{\partial x}$ and the analytical solution of $x$, we have
\begin{equation}
\begin{aligned}
    \frac{\partial f(x,C'_{\text{real}})}{\partial x}&=-{C'}_{\text{real}}^TM_{\text{real}} \text{diag}(\gamma)(C'_{\text{real}}x-d')-{C'}_{\text{real}}^T(\frac{1}{4K}M_{\text{real}}+U_{\text{real}})\gamma\\
    x&=(C'^TM \text{diag}(\gamma)C')^{-1}(C'^TM \text{diag}(\gamma)d'-C'^T(\frac{M}{4K}+U)\gamma)=P\beta\\
    \end{aligned}
    \end{equation}
    According to the analytical solution of optimal point $x$, the derivative $\frac{\partial x_{i}}{\partial C'_{k,l}}$ is
    \begin{equation}
    \frac{\partial x_{i}}{\partial C'_{k,l}}=\sum_{j}\frac{\partial P_{i,j}}{\partial C'_{k,l}}\beta_j+\sum_{j}P_{i,j}\frac{\partial\beta_j}{\partial C'_{k,l}}
    \label{RP:derivative}
    \end{equation}
    For the first term of the derivative above, we have:
    \begin{equation}
    \begin{aligned}
    \frac{\partial P_{i,j}}{\partial C'_{k,l}}&=\frac{\partial (C'^TM \text{diag}(\gamma)C')_{i,j}^{-1}}{C'_{k,l}}\\
    &=\sum_{p}\sum_{q}\frac{\partial (C'^TM \text{diag}(\gamma)C')_{p, q}}{\partial C'_{k, l}}\frac{P_{i, j}}{P^{-1}_{p, q}}\\
    &=\sum_{p}\sum_{q}\frac{\partial \sum_{x}\sum_{y}C'_{x,p}M \text{diag}(\gamma)_{x,y}C'_{y,q}}{\partial C'_{k,l}}(-P_{i,p}P_{q,j})\\
    &=\sum_{q}\sum_{y}[p==l]M \text{diag}(\gamma)_{k,y}C'_{y,q}(-P_{i,p}P_{q,j})+\sum_{p}\sum_{x}[q==l]M \text{diag}(\gamma)_{x,k}C'^T_{p,x}(-P_{i,p}P_{q,j})\\
    &=\sum_{q}-P_{i,l}(M \text{diag}(\gamma))_{k,*}C'_{*,q}P_{q,j}+\sum_{p}-P_{l,j}C'^T_{p,*}(M \text{diag}(\gamma))_{*,k}P_{i,p}\\
    &=-(P_{i,l}(M \text{diag}(\gamma))_{k,*}C'P_{*,j}+P_{l,j}P_{i,*}C'^T(M \text{diag}(\gamma))_{*,k})\\
    &=-(P_{i,l}(M \text{diag}(\gamma)C'P)_{k,j}+P_{l,j}(PC'^TM(\text{diag}(\gamma)))_{i,k})
    \end{aligned}
    \end{equation}
    Therefore, the simplified result for the first term of the derivative in Equation \ref{RP:derivative} is:
    \begin{equation}
       \sum_{j}\frac{\partial P_{i,j}}{\partial C'_{k,l}}\beta_j=-((M \text{diag}(\gamma)C'P\beta)_kP_{i,l}+(P\beta)_l(PC'^TM \text{diag}(\gamma))_{i,k})
    \end{equation}
For the second term, we have:
\begin{equation}
\frac{\partial \beta_j}{\partial C'_{k,l}}=\frac{\partial (C'^T\eta)_j}{C'_{k,l}}
=\frac{\partial \sum_{p}C'_{p,j}\eta_p}{\partial C'_{k,l}}
= \eta_k[j==l]
\end{equation}
The simplified second term is thus
\begin{equation}
   \sum_{j}P_{i,j}\frac{\partial\beta_j}{\partial C'_{k,l}}=P_{i,l}\eta_{k} 
\end{equation}
Finally, the derivative can be written as
\begin{equation}
\frac{\partial x_i}{\partial C'_{k,l}}=-((M \text{diag}(\gamma)C'P\beta)_kP_{i,l}+(P\beta)_l(PC'^TM \text{diag}(\gamma))_{i,k})+P_{i,l}\eta_{k}
\end{equation}

\section{Benchmark Details : Dataset and Problem Settings}
\label{app:benchmark-details}
Our code is public in  \href{https://github.com/PredOptwithSoftConstraint/PredOptwithSoftConstraint}{the repo: } https://github.com/PredOptwithSoftConstraint/PredOptwithSoftConstraint.

\subsection{Synthetic Linear Programming}
\label{subapp:app-prob1}

\subsubsection{Prediction Dataset}\label{sub:pred_dataset}
We generate the synthetic dataset $\{\xi_i, \theta_i\}_{i=1}^N$ under a general structural causal model (\cite{peters2017elements}). In the original form, it is like:
\begin{align}
\begin{split}
z \sim N(0, \Sigma) \\
\xi = g(z) + \epsilon_1 \\
\theta = h(z) + \epsilon_2 \\
\end{split}
\end{align}
where $z$ is the latent variable, $\xi$ observed features of $z$, and $\theta$ the result variable caused by $z$. According to physical knowledge, $h$ can be a process of linear, quadratic, or bi-linear form. However, it is difficult to get an explicit form of reasonable $g$. Instead, $g^{-1}$ can be well-represented by deep neural networks.

Thus, alternatively, we use the following generative model:
\begin{align}
\begin{split}
\xi^* \sim N(0, \Sigma) \\
z = m(\xi^*) \\
\theta = h(z) + \epsilon_2 \\
\xi = \xi^* + \epsilon_1
\end{split}
\end{align}
where $m$ behaves as $g^{-1}$. In our experiment settings, $\Sigma=I+QQ^T$, where each element of $Q$ is generated randomly at uniform from $(0, 1)$. We set $m(x)=\sin(2\pi xB)$, where $B$ is a matrix whose elements are generated randomly at uniform in $\{0,1\}$, and $\sin$ is applied element-wisely. We implement $h(z)$ as a MLP with two hidden layers, and the output is normalized to $(0, 1]$ for each dimension through different data points. Finally, we add a noise of $0.01 \epsilon_x$ to $x$, where $\epsilon_x \sim N(0,1)$; and $0.01 \epsilon_{\theta}$ to $\theta$, where $\epsilon_{\theta}$ follows a truncated normal distribution which truncates a normal distribution $N(0,1)$ to $[0,1.5]$.

The dataset is split into training, validation and test sets with the proportions $50\%$, $25\%$, $25\%$ in respect. The batch size is set to $10$ for $N=100$, $50$ for $N=1000$, and $125$ for $N=5000$.

\subsubsection{Problem Settings}

We generate hard constraint $Ax\leq b$, and soft constraint $Cx\leq d$. Each element of $A$ or $C$ is first generated randomly at uniform within $(0, 1)$, then set to $0$ with probability of $0.5$. We generate $b,d$ as $b=0.5A\mathbf{1}$ and $d=0.25C\mathbf{1}$. The soft constraint coefficient $\alpha$ is generated randomly at uniform from $(0, 0.2)$ for each dimension. 

\subsection{Portfolio Optimization}

\subsubsection{Dataset}
The prediction dataset is daily price data of SP500 from 2004 to 2017 downloaded by Quandl API \cite{sp500-dataset} with the same settings in \cite{surrogate-melding}. Most settings are aligned with those in \cite{surrogate-melding}, including dataset configuration, prediction model, learning rate (initial $0.01$ with scheduler), optimizer (Adam), gradient clip ($0.01$), the number of training epochs ($20$), and the problem instance size (the number of equities being $\{50, 100, 150, 200, 250\}$).

\subsubsection{Problem Settings}
We set the number of soft constraints to $0.4$ times of $n$, where $n$ is the number of candidate equities. For the soft constraint $\alpha^T\max(Cx-d, 0)$, $\alpha=\frac{15}{n}v$, where each element of $v$ is generated randomly at uniform from $(0, 1)$; the elements of matrix $C$ are generated independently from $\{0, 1\}$, where the probability of $0$ is $0.9$ and $1$ is $0.1$. $K$ is set as $100$.

\subsection{Resource Provisioning}
\label{subapp:app:prob3}

\subsubsection{Dataset}
The ERCOT energy dataset \cite{ercot-dataset} contains hourly data of energy output from 2013 to 2018 with $52535$ data points. We use the first $70\%$ as the training set, the middle $10\%$ as the validation set, and the last $20\%$ as the test set. We normalize the dataset by dividing the labels by $10^4$, which makes the typical label becomes value around $(0.1, 1)$. We train our model with a batch size of $256$; each epoch contains $144$ batches. We aim to predict the matrix $C\in\mathbb{R}^{24\times 8}$, where $24$ represents the following $24$ hours and $8$ represents the $8$ regions, which are $\{\text{COAST, EAST, FWEST, NCENT, NORTH, SCENT, SOUTH, WEST}\}$. The data is drawn from the dataset of the corresponding region. the decision variable $x$ is $8$-dimensional, and $d=0.5\mathbf{1}+0.1N(0, 1)$. 

\subsubsection{Problem Settings}
We test five sets of $(\alpha_1, \alpha_2)$, which are $(50\times\mathbf{1}, 0.5\times\mathbf{1})$, $(5\times\mathbf{1}, 0.5\times\mathbf{1})$, $(\mathbf{1},\mathbf{1})$, $(0.5\times\mathbf{1}, 5\times\mathbf{1})$,  and $(0.5\times\mathbf{1}, 50\times\mathbf{1})$, against two-stage with $L1$-loss, $L2$-loss, and weighted $L1$-loss. We use AdaGrad as optimizer with learning rate $0.01$, and clip the gradient with norm $0.01$. The feature is a $(8\times 24\times 77)$-dimensional vector for each matrix $C$; we adopted McElWee's Blog \cite{kevein-blog} for the feature generation and the model. Weighted L1-loss has the following objective:
\begin{equation}
    \alpha_2^T\max(Cx-d, 0) + \alpha_1^T\max(d-Cx, 0)
\end{equation}
Note that $\alpha_2$ and $\alpha_1$ are exchanged in the objective. Intuitively, this is because an under-estimation of entries of $C$ will cause a larger solution of $x$, which in turn makes $C_{\text{real}}x-d$ larger at test time, and vice versa.
Another thing worth noting is that SPO+ cannot be applied to the prediction of the matrix $C$ or vector $d$, for it is designed for the scenario where the predicted parameters are in the objective. 

\REM{
In the synthetic linear programming experiment where the predicted parameter is $\theta$, we change $z=Cx-d$ to $z\geq\max(Cx-d, 0)$ and put $z$ into the decision variables, which makes the problem becoming
\begin{equation}
    \max \begin{bmatrix}\theta\\0\end{bmatrix}^T\begin{bmatrix}x\\z\end{bmatrix},\ s.t. \begin{bmatrix}A&0\\-I&0\\0&-I\\C&-I\end{bmatrix}\begin{bmatrix}x\\z\end{bmatrix}\leq\begin{bmatrix}b\\0\\0\\d\end{bmatrix}
\end{equation}
Note that by this conversion, $C$ and $d$ appear in the constraint, which makes the prediction+optimization over such variables infeasible. Also, the objective derivative upper bound $E$ for this question is at most $O(E_0\max(||\alpha_1||_2, ||\alpha_2||_2))$, where $E_0$ is the upper bound of the Frobenius norm of matrix $C$ (\ie{} $||C||_F$). 
}

\section{Supplementary Experiment Results}

%
\REM{
We provide supplementary results of our experiment in this section; they are part of the analysis of our empirical evaluation of methods, and are not appearing in the main paper due to the space limit.}

\subsection{Linear Programming with Soft Constraints}\label{subapp:lp-result}

\textbf{The effect of $K$.} Table \ref{tab:allK} provides the mean and standard deviation of regrets under varying values of $K$. \REM{Although there are no explicit rules on the formula of $K$ with respect to problem size or $N$,} we can see that our method performs better than all other methods with most settings of $K$. 

Empirically, to find an optimal $K$ for a given problem and its experimental settings, a grid search with roughly adaptive steps suffices. For example,in our experiments, we used a proposal where neighbouring coefficient has $5x$ difference (e.g. $\{0.2, 1, 5, 25, 125\}$) works well. Quadratic objective, as in the second experiment, usually requires larger $K$ than linear objective as in the first and the third experiment.

\begin{table}[h]
    \centering
    \label{tab:allK}
    \begin{tabular}{ccccccc}
    & & \multicolumn{4}{c}{Regret}\\
    \hline 
    $N$ & Problem Size & ours($K=0.2$) & $K=1.0$ & $K=5.0$ & $K=25.0$ & $K=125.0$\\
    \hline
        100 & (40, 40, 0) & 2.423$\pm$0.305 & 2.378$\pm$0.293 & 2.265$\pm$0.238 & 2.301$\pm$0.325 & \textbf{2.258$\pm$0.311} \\
            & (40, 40, 20) & 2.530$\pm$0.275 & 2.574$\pm$0.254 & 2.384$\pm$0.277 & 2.321$\pm$0.272 & \textbf{2.350$\pm$0.263} \\
            & (80, 80, 0) & \textbf{5.200$\pm$0.506} & 5.529$\pm$0.307 & 5.610$\pm$0.351 & 5.563$\pm$0.265 & 5.502$\pm$0.353\\
            & (80, 80, 40) & \textbf{4.570$\pm$0.390} & 4.747$\pm$0.470 & 4.758$\pm$0.531 & 4.780$\pm$0.450 & 4.796$\pm$0.624 \\
    \hline 
       1000 & (40, 40, 0) & 1.379$\pm$0.134 & 1.387$\pm$0.130 & 1.365$\pm$0.147 & \textbf{1.346$\pm$0.144} & 1.352$\pm$0.140 \\
            & (40, 40, 20) & 1.587$\pm$0.112 & 1.553$\pm$0.114 & 1.523$\pm$0.109 & 1.507$\pm$0.102 & \textbf{1.506$\pm$0.102}\\
            & (80, 80, 0) & 3.617$\pm$0.125 & \textbf{3.431$\pm$0.100} & 3.540$\pm$0.105 & 3.513$\pm$0.097 & 3.507$\pm$0.117 \\
            & (80, 80, 40) & \textbf{2.781$\pm$0.165} & 2.817$\pm$0.167 & 2.819$\pm$0.155 & 2.820$\pm$0.157 & 2.826$\pm$0.197\\
    \hline
       5000 & (40, 40, 0) & 1.041$\pm$0.094 & 1.065$\pm$0.098 & 1.045$\pm$0.099 & 1.038$\pm$0.102 & \textbf{1.037$\pm$0.100}\\
            & (40, 40, 20) & 1.278$\pm$0.072 & 1.248$\pm$0.076 & 1.223$\pm$0.075 & \textbf{1.220$\pm$0.071} & 1.221$\pm$0.074\\
            & (80, 80, 0) & 3.074$\pm$0.098 & \textbf{2.845$\pm$0.064} & 2.884$\pm$0.100 & 2.889$\pm$0.083 & 2.864$\pm$0.086\\
            & (80, 80, 40) & 2.217$\pm$0.123 & 2.222$\pm$0.111 & \textbf{2.172$\pm$0.098} & 2.174$\pm$0.113 & 2.177$\pm$0.105\\
    \hline
    \end{tabular}
    \caption{The mean and standard deviation of the regret of all $K$ in our experiment.}
    \label{tab:allK}
\end{table}
\textbf{Prediction error.} Table \ref{tab:MSE1} and Table \ref{tab:MSE2} are the mean and standard deviation of prediction MSE for all methods. We can find that two-stage methods have an advantage over end-to-end methods on prediction error. however, as demonstrated in our experiments, such advantage does not necessarily lead to the advantage on the regret performance. Surprisingly, although DF\cite{aaai-WilderDT19} works generally on par with other baselines in regret performance, it generates a very large MSE error compared with our method; closer examination suggests that there are some outlier cases where the prediction loss is magnitudes higher than others. Also, the prediction error of DF gets much higher when no soft constraint exists, which is probably because the optimal solution remains the same when all parameters to predict are scaled by a constant factor.

\begin{table}
    \centering
    
    \begin{tabular}{ccccccc}
    & & \multicolumn{5}{c}{MSE of $\theta$}\\
    \hline 
    $N$ & Problem Size & L1 & L2 & SPO+ & DF & ours ($K=0.2$)\\
    \hline
        100 & (40, 40, 0) & \textbf{0.062$\pm$0.006} & 0.063$\pm$0.008 & 0.079$\pm$0.150 & 14.592$\pm$9.145 & 0.094$\pm$0.039 \\
            & (40, 40, 20) & 0.061$\pm$0.005 & 0.062$\pm$0.006 & 0.071$\pm$0.006 & 6.414$\pm$4.451& \textbf{0.061$\pm$0.007}\\
            & (80, 80, 0) & \textbf{0.058$\pm$0.004} & 0.059$\pm$0.004 & 0.066$\pm$0.011 & 28.768$\pm$13.90& 0.247$\pm$0.009\\
            & (80, 80, 40) & \textbf{0.059$\pm$0.004} & 0.059$\pm$0.004 & 0.085$\pm$0.008 & 8.98$\pm$ 9.99& 0.061$\pm$0.005 \\
    \hline 
       1000 & (40, 40, 0) & 0.031$\pm$0.001 & \textbf{0.031$\pm$0.001} & 0.034$\pm$0.002 &34.931$\pm$14.309 & 0.036$\pm$0.003 \\
            & (40, 40, 20) & \textbf{0.031$\pm$0.001} & 0.031$\pm$0.001 & 0.033$\pm$0.001 & 9.255$\pm$8.497 & 0.033$\pm$0.001\\
            & (80, 80, 0) & 0.030$\pm$0.001 & \textbf{0.029$\pm$0.001} & 0.032$\pm$0.002 & 90.590$\pm$35.298 & 0.222$\pm$0.009 \\
            & (80, 80, 40) & 0.030$\pm$0.001 & \textbf{0.029$\pm$0.001} & 0.033$\pm$0.001 & 1.060$\pm$2.395 & 0.033$\pm$0.002\\
    \hline
       5000 & (40, 40, 0) & \textbf{0.022$\pm$0.000} & 0.022$\pm$0.000 & 0.025$\pm$0.001 & 72.083$\pm$46.685& 0.025$\pm$0.001\\
            & (40, 40, 20) & 0.022$\pm$0.000 & \textbf{0.022$\pm$0.000} & 0.238$\pm$0.001 & 235.789$\pm$842.939& 0.243$\pm$0.001\\
            & (80, 80, 0) & 0.021$\pm$0.000 & \textbf{0.021$\pm$0.000} & 0.023$\pm$0.001 & 147.144$\pm$51.101& 0.196$\pm$0.011\\
            & (80, 80, 40) & 0.020$\pm$0.000 & \textbf{0.020$\pm$0.000} & 0.022$\pm$0.001 & 0.372$\pm$0.311 & 0.024$\pm$0.001\\
    \hline
    \end{tabular}
    \caption{The mean and standard deviation of the MSE loss of the predicted variable $\theta$ in the synthetic experiment on the test set. The two-stage methods have better performance on the accuracy of prediction.}
    \label{tab:MSE1}
\end{table}

\begin{table}
    \centering
    \begin{tabular}{cccccc}
    & & \multicolumn{4}{c}{MSE of $\theta$}\\
    \hline 
    $N$ & Problem Size & ours($K=1.0$) & ours($K=5.0$) & ours($K=25.0$) & ours($K=125.0$)\\
    \hline
        100 & (40, 40, 0) & 0.068$\pm$0.008 & 0.091$\pm$0.037 & 0.188$\pm$0.089 & 0.204$\pm$0.079 \\
            & (40, 40, 20) & 0.094$\pm$0.031 & 0.204$\pm$0.074 & 0.244$\pm$0.105 & 0.343$\pm$0.200  \\
            & (80, 80, 0) & 0.129$\pm$0.049 & 0.257$\pm$0.182 & 0.641$\pm$0.476 & 0.816$\pm$0.351\\
            & (80, 80, 40) & 0.094$\pm$0.013 & 0.151$\pm$0.035 & 0.172$\pm$0.091 & 0.177$\pm$0.042\\
    \hline 
       1000 & (40, 40, 0) & 0.036$\pm$0.002 & 0.056$\pm$0.010 & 0.143$\pm$0.048 & 0.195$\pm$0.095\\
            & (40, 40, 20) & 0.060$\pm$0.010 & 0.127$\pm$0.051 & 0.172$\pm$0.071 & 0.193$\pm$0.078\\
            & (80, 80, 0) & 0.114$\pm$0.070 & 0.220$\pm$0.097 & 1.124$\pm$0.341 & 1.407$\pm$0.439\\
            & (80, 80, 40) & 0.039$\pm$0.002 & 0.060$\pm$0.006 & 0.078$\pm$0.010 & 0.080$\pm$0.008 \\
    \hline
       5000 & (40, 40, 0) & 0.026$\pm$0.001 & 0.033$\pm$0.004 & 0.076$\pm$0.025 & 0.089$\pm$0.037\\
            & (40, 40, 20) & 0.036$\pm$0.004 & 0.073$\pm$0.019 & 0.110$\pm$0.038 & 0.113$\pm$0.026\\
            & (80, 80, 0) & 0.152$\pm$0.079 & 0.097$\pm$0.040 &  0.703$\pm$0.177 & 0.895$\pm$0.205 \\
            & (80, 80, 40) & 0.026$\pm$0.001 & 0.032$\pm$0.003 & 0.038$\pm$0.004 & 0.040$\pm$0.007 \\
    \hline
    \end{tabular}
    \caption{The mean and standard deviation of the MSE loss of the predicted variable $\theta$ in the synthetic experiment for different values of $K$.}
    \label{tab:MSE2}
\end{table}

\textbf{Detailed behaviour of SPO+.} As mentioned in the main paper, SPO+ quickly becomes overfited in our experiments, shown by Table \ref{tab:epoches} as empirical evidence. In our experiments, early stopping starts at epoch $8$, and maximum number of run epochs is $40$; the results show that SPO+ stops much earlier than other methods under most settings. Figure \ref{fig:overfit} plots the timely comparison of test performance during training, where SPO+ get overfited quickly, although it performs better in the earlier period of training.

\begin{figure}[h]
    \centering
    \includegraphics[width=6cm]{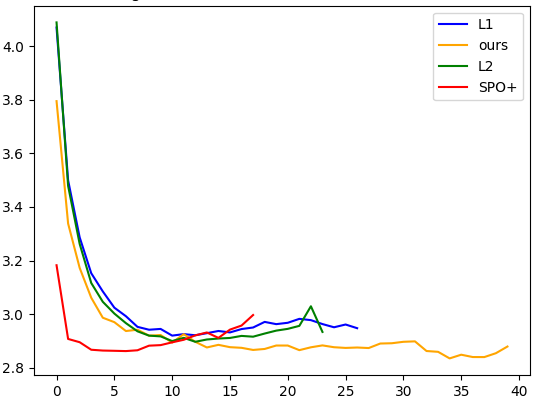}
    \caption{The average regret on test set of $K=25.0$, training set size $5000$ and problem size $(80, 80, 0)$ with respect to the number of epochs (methods other than ours does not reach $40$ in this figure, for all runs are early-stopped before epoch $40$). While SPO+ is better than two-stage in the best performance, it overfits rather quickly.}
    \label{fig:overfit}
\end{figure}

\begin{table}[h]
    \centering
    \begin{tabular}{ccccccc}
    & & \multicolumn{5}{c}{Average Episode}\\
    \hline 
    $N$ & Problem Size & L1 & L2 & SPO+ & DF & ours ($K=0.2$)\\
    \hline
        100 & (40, 40, 0) & 20.46$\pm$10.20 & 12.53$\pm$3.00 & 19.6$\pm$8.58 &29.53$\pm$10.24& 20.53$\pm$10.39 \\
            & (40, 40, 20) & 21.47$\pm$9.58 & 14.13$\pm$4.47 & 17.8$\pm$9.98 &25.46$\pm$11.41& 22.47$\pm$10.29\\
            & (80, 80, 0) & 22.6$\pm$5.78 & 22.6$\pm$8.52 & 20.07$\pm$7.23 &30.87$\pm$8.83& 24.53$\pm$8.02\\
            & (80, 80, 40) & 25.2$\pm$8.38 & 24.73$\pm$8.88 & 19.33$\pm$764 &24$\pm$9.06& 28.33$\pm$8.89 \\
    \hline 
       1000 & (40, 40, 0) & 14.67$\pm$2.61 & 13.2$\pm$1.74 & 10$\pm$2.04 & 27.93$\pm$7.42 & 13.6$\pm$2.53 \\
            & (40, 40, 20) & 13.2$\pm$2.37 & 13.4$\pm$2.50 & 10.33$\pm$1.40 & 24.6$\pm$8.14 & 10.67$\pm$1.23\\
            & (80, 80, 0) & 19.07$\pm$2.25 & 18.07$\pm$1.83 & 11.87$\pm$2.00 & 25.2$\pm$6.41 & 24.27$\pm$11.02 \\
            & (80, 80, 40) & 19.93$\pm$2.66 & 19$\pm$1.31 & 12.33$\pm$1.95 &25.4$\pm$9.22& 17.73$\pm$4.08\\
    \hline
       5000 & (40, 40, 0) & 16.47$\pm$2.59 & 16.27$\pm$1.33 & 11.87$\pm$2.10 & 23.33$\pm$7.34 & 11.8$\pm$2.60\\
            & (40, 40, 20) & 17.13$\pm$2.59 & 15.27$\pm$1.79 & 12.8$\pm$2.88 & 21.67$\pm$7.95 & 13.4$\pm$2.77\\
            & (80, 80, 0) & 20.67$\pm$3.20 & 20.8$\pm$2.18 & 13.7$\pm$2.16 & 29.33$\pm$7.22 & 30.07$\pm$9.48\\
            & (80, 80, 40) & 20.47$\pm$2.64 & 21$\pm$2.88 & 12.2$\pm$1.61 & 17.6$\pm$6.84& 18.53$\pm$5.26\\
    \hline
    \end{tabular}
    \caption{The mean and standard deviation of the number of epochs run in each instance (capped at $40$). Note that SPO+ is significantly more prone to overfitting in our experiment settings.}
    \label{tab:epoches}
\end{table}

\textbf{Statistical significance tests.} Table \ref{tab:pvalue01} shows the results of significance test (one-tailed paired $t$-test), with the assumption that the means of two distributions are the same. The results show that our method is significantly better than other methods under almost all settings.

\begin{table}[h]
    \centering
    \begin{tabular}{cccccc}
    & & \multicolumn{4}{c}{$t$-value}\\
    \hline 
    $N$ & Problem Size & L1 & L2 & SPO+ & DF\\
    \hline
        100 & (40, 40, 0) & 3.884 & 4.615 & 5.194 & 1.614\\
            & (40, 40, 20) & 5.173 & 5.248 & 6.836 & 1.569\\
            & (80, 80, 0) & 6.024 & 5.661 & 4.523 & 3.610\\
            & (80, 80, 40) & 2.308 & 1.724 & 3.622 & 1.936\\
    \hline 
       1000 & (40, 40, 0) & 8.438 & 5.271 & 4.860 & 1.115\\
            & (40, 40, 20) & 6.547  & 6.212 & 8.562 & 2.620\\
            & (80, 80, 0) & 12.620 & 10.925 & 7.979 & 2.751\\
            & (80, 80, 40) & 7.116 & 4.465 & 6.223 & 8.230\\
    \hline
       5000 & (40, 40, 0) & 7.130 & 7.808 & 7.546 & 1.175\\
            & (40, 40, 20) & 10.410 & 10.238 & 6.555 & 2.382\\
            & (80, 80, 0) & 8.361  & 7.502  & 6.874 & 0.869\\
            & (80, 80, 40) & 6.550 & 4.388 & 5.066 & 11.620\\
    \hline
    \end{tabular}
    \caption{The $t$-value of the one-tailed paired $t$-test between all other methods and our methods under optimal $K$.}
    \label{tab:pvalue01}
\end{table}

\begin{table}[h]
    \centering
    \begin{tabular}{cccccc}
    & &  \multicolumn{4}{c}{$p$-value}\\
    \hline 
    $N$ & Problem Size & L1 & L2 & SPO+ & DF \\
    \hline
        100 & (40, 40, 0) & 8.272 $\times 10^{-4}$ & 2.006$\times 10^{-4}$ & 6.809 $\times 10^{-5}$ & 0.058\\
            & (40, 40, 20) & 7.076$\times 10^{-5}$ & 6.162$\times 10^{-5}$ & 4.054$\times 10^{-6}$ & 0.064\\
            & (80, 80, 0) & 1.561$\times 10^{-5}$ & 2.936$\times 10^{-5}$ & 2.388$\times 10^{-4}$ & 5$\times 10^{-4}$\\
            & (80, 80, 40) & 0.018 & 0.053 & 0.001 & 0.032\\
    \hline 
       1000 & (40, 40, 0) & 3.662$\times 10^{-7}$ & 5.917$\times 10^{-5}$ & 1.263 $\times 10^{-4}$ & 0.137\\
            & (40, 40, 20) & 6.489$\times 10^{-6}$ & 1.133$\times 10^{-5}$ & 3.078$\times 10^{-7}$ & 0.007\\
            & (80, 80, 0) & 2.444$\times 10^{-9}$ & 1.544$\times 10^{-8}$ & 7.064$\times 10^{-7}$ & 0.005\\
            & (80, 80, 40) & 2.600$\times 10^{-6}$ & 2.667$\times 10^{-4}$ & 1.114$\times 10^{-5}$ & 2.937$\times 10^{-9}$\\
    \hline
       5000 & (40, 40, 0) & 2.545$\times 10^{-6}$ & 9.074$\times 10^{-7}$ & 1.342$\times 10^{-6}$ & 0.125\\
            & (40, 40, 20) & 2.834$\times 10^{-8}$& 3.487$\times 10^{-8}$ & 6.406$\times 10^{-6}$ & 0.012\\
            & (80, 80, 0) & 4.082$\times 10^{-7}$ & 1.435$\times 10^{-6 }$ & 3.818$\times 10^{-6}$ & 0.196\\
            & (80, 80, 40) & 6.460$\times 10^{-6}$ & 3.097$\times 10^{-4}$ & 8.606$\times 10^{-5} $ & 1.585$\times 10^{-12}$\\
    \hline
    \end{tabular}
    \caption{The $p$-value of the one-tailed paired $t$-test between all other methods and our methods under optimal $K$.}
    \label{tab:pvalue01}
\end{table}

\subsection{Portfolio Optimization}

\textbf{Statistical significance tests.} Table \ref{tab:pvalue02} shows the result of significance test (one-tailed paired $t$-test), with the assumption that the means of two distributions are the same. The results show that our method is significantly better than other methods under every setting.

\begin{table}[h]
    \centering
    \begin{tabular}{ccccccc}
          & \multicolumn{3}{c}{$t$-value} & \multicolumn{3}{c}{$p$-value}\\
          \hline
         \#Equities & L1 & L2 & DF & L1 & L2 & DF \\
         \hline
        50 & 10.329 & 10.628 & 6.141 & 3.213$\times 10^{-8}$ & 2.186$\times 10^{-8}$ & 1.279$\times 10^{-5}$\\
        100 & 21.414 & 23.995 & 6.200 & 2.125$\times 10^{-12}$ & 4.495$\times 10^{-13}$ & 1.157$\times 10^{-5}$\\
        150 & 13.402 & 14.503 & 5.690 & 1.119$\times 10^{-9}$ & 3.971$\times 10^{-10}$ & 2.789$\times 10^{-5}$\\
        200 & 13.617 & 13.104 & 4.329 & 9.086$\times 10^{-10}$ & 1.500$\times 10^{-9}$ & 3.465$\times 10^{-4}$\\
        250 & 13.904 & 14.483 & 3.205 &  6.915$\times 10^{-10}$ & 4.044$\times 10^{-10}$ & 0.003  \\
    \hline
    \end{tabular}
    \caption{The $t$-value and $p$-value of the one-tailed paired $t$-test between all other methods and our methods.}
    \label{tab:pvalue02}
\end{table}

\subsection{Resource Provisioning}

\textbf{Statistical significance tests.} Table \ref{tab:pvalue03} shows the result of significance test (one-tailed paired $t$-test), with the assumption that the means of two distributions are the same.

\begin{table}[h]
    \centering
    \begin{tabular}{ccccccc}
     & \multicolumn{3}{c}{$t$-value} & \multicolumn{3}{c}{$p$-value}\\
          \hline
        $\alpha_1/\alpha_2$ & L1 & L2 & Weighted L1 & L1 & L2 & Weighted L1\\
        \hline
        100 & 13.133 & 9.357 & 6.429 & 2.914$\times 10^{-9}$& 2.115$\times 10^{-7}$ & 1.576$\times 10^{-5}$  \\
        10 & 12.484 & 14.231 & 10.897 & 5.622$\times 10^{-9}$ & 1.493$\times10^{-6}$ & 3.190 $\times 10^{-8}$ \\
        0.1 & 1.184 & 5.073 & 11.001 & 0.127 & 8.486 $\times 10^{-5}$ & 1.416 $\times 10^{-8}$ \\
        0.01 & 0.535 & 1.998 & 5.083 & 0.300 & 0.032 & 8.33$\times 10^{-5}$ \\
    \hline
    \end{tabular}
    \caption{The $t$-value ($p$-value) of the one-tailed paired $t$-test between all other methods and our methods.}
    \label{tab:pvalue03}
\end{table}

\section{Computing Infrastructure}
All experiments are conducted on Linux Ubuntu 18.04 bionic servers with 256G memory and 1.2T disk space with no GPU, for GPU does not suit well with Gurobi.\footnote{see Gurobi's official support website: https://support.gurobi.com/hc/en-us/articles/360012237852-Does-Gurobi-support-GPUs-} Each server has 32 CPUs, which are Intel Xeon Platinum 8272CL @ 2.60GHz.

For the first experiment, we use a few minutes to get one set of data \footnote{Running all methods simultaneously under one particular parameter setting.} with training set size $100$, and about $4-5$ hours to get one set of data with training set size $5000$. For the second experiment, we use about an hour to get one set of data with $N=50$, and $2-3$ days to get one set of data with $N=250$. For the third experiment, we use around $6-7$ hours to obtain one set of data. Though our method is slower to train than two-stage methods, it is 2-3x faster to train than KKT-based decision-focused method.

\end{document}